%% file: GBM_arixv.tex
\documentclass[11pt]{article}

\usepackage{epsfig,epsf,fancybox}
\usepackage{amsmath}
\usepackage{mathrsfs}
\usepackage{amssymb}
\usepackage{graphicx}
\usepackage{color}
\usepackage{multirow}
\usepackage{paralist}
\usepackage{verbatim}
\usepackage{galois}
\usepackage{algorithm}
\usepackage[noend]{algorithmic}
\usepackage{boxedminipage}
\usepackage{booktabs}
\usepackage{accents}
\usepackage{stmaryrd}
\usepackage[table]{xcolor}
\usepackage{hhline}

\usepackage{subfig}

\usepackage{natbib}

\usepackage{url}%
\usepackage[colorlinks,linkcolor=magenta,citecolor=blue, pagebackref=true,backref=true]{hyperref}
\renewcommand*{\backrefalt}[4]{%
    \ifcase #1 \footnotesize{(Not cited.)}%
    \or        \footnotesize{(Cited on page~#2.)}%
    \else      \footnotesize{(Cited on pages~#2.)}%
    \fi}

\newtheorem{corcou}{Corollary}
\newtheorem{cor}[corcou]{Corollary}
\newtheorem{lemcou}{Lemma}
\newtheorem{lem}[lemcou]{Lemma}

\newtheorem{defncou}{Definition}
\newtheorem{defn}[defncou]{Definition}
\newtheorem{thmcou}{Theorem}
\newtheorem{thm}[thmcou]{Theorem}

\newtheorem{remcou}{Remark}
\newtheorem{rem}[remcou]{Remark}

\numberwithin{equation}{section}

\title{\bf{\LARGE{GBM-based Bregman Proximal Algorithms for Constrained Learning}}}

\author{Zhenwei Lin\thanks{zhenweilin@163.sufe.edu.cn, Shanghai University of Finance and Economics}$\qquad$   $\qquad$  Qi Deng \thanks{qideng@sufe.edu.cn, Shanghai University of Finance and Economics}
}

\usepackage{bm}
\global\long\def\apd{\textup{ABPP}}%
\global\long\def\ilcp{\textup{CBPR}}%
\global\long\def\bell{\bm{\ell}}%

\include{macros}

\begin{document}
\maketitle

\input{GBM_arxiv_main}

\bibliographystyle{plainnat}
\bibliography{ref}

\setcounter{thmcou}{0}
\setcounter{propcou}{2}
\setcounter{lemcou}{0}
\setcounter{corcou}{0}
\setcounter{defncou}{3}
\appendix
\newpage
\input{supp_main}

\end{document}

%% file: macros.tex
\global\long\def\inprod#1#2{\left\langle #1,#2\right\rangle }%

\global\long\def\binner#1#2{\big\langle#1,#2\big\rangle}%

\global\long\def\norm#1{\Vert#1\Vert}%

\global\long\def\brbra#1{\big(#1\big)}%

\global\long\def\sbra#1{[#1]}%

\global\long\def\abs#1{\vert#1\vert}%

\global\long\def\bcbra#1{\big\{#1\big\}}%

\global\long\def\vertiii#1{\left\vert \kern-0.25ex  \left\vert \kern-0.25ex  \left\vert #1\right\vert \kern-0.25ex  \right\vert \kern-0.25ex  \right\vert }%

\global\long\def\mcal#1{\mathcal{#1}}%

\global\long\def\mbb#1{\mathbb{#1}}%

\global\long\def\onebf{\mathbf{1}}%

\global\long\def\zerobf{\mathbf{0}}%

\global\long\def\dist{\operatornamewithlimits{dist}}%

\global\long\def\argmin{\operatornamewithlimits{argmin}}%

\global\long\def\st{\operatornamewithlimits{s.t.}}%

\global\long\def\dom{\mathrm{dom}}%

\global\long\def\and{\mathrm{and}}%

\global\long\def\vep{\varepsilon}%

\global\long\def\dom{\operatornamewithlimits{dom}}%

\global\long\def\Rbb{\mathbb{R}}%

\global\long\def\Ocal{\mathcal{O}}%

\global\long\def\Xcal{\mathcal{X}}%

\global\long\def\abf{\mathbf{a}}%

\global\long\def\bbf{\mathbf{b}}%

\global\long\def\qbf{\mathbf{q}}%

\global\long\def\gbf{\mathbf{g}}%

\global\long\def\ebf{\mathbf{e}}%

\global\long\def\xbf{\mathbf{x}}%

\global\long\def\ybf{\mathbf{y}}%

\global\long\def\zbf{\mathbf{z}}%

\global\long\def\Wbf{\mathbf{W}}%

\global\long\def\Abf{\mathbf{A}}%

\global\long\def\Ibf{\mathbf{I}}%

\global\long\def\Dbf{\mathbf{D}}%

%% file: GBM_arxiv_main.tex
\begin{abstract}
As the complexity of learning tasks surges, modern machine learning encounters a new constrained learning paradigm characterized by more intricate and data-driven function constraints. Prominent applications include Neyman-Pearson classification (NPC) and fairness classification, which entail specific risk constraints that render standard projection-based training algorithms unsuitable. Gradient boosting machines (GBMs) are among the most popular algorithms for supervised learning; however, they are generally limited to unconstrained settings. In this paper, we adapt the GBM for constrained learning tasks within the framework of Bregman proximal algorithms.  We introduce a new Bregman primal-dual method with a global optimality guarantee when the learning objective and constraint functions are convex. In cases of nonconvex functions, we demonstrate how our algorithm remains effective under a Bregman proximal point framework. Distinct from existing constrained learning algorithms, ours possess a unique advantage in their ability to seamlessly integrate with publicly available GBM implementations such as XGBoost (Chen and Guestrin, 2016) and LightGBM (Ke et al., 2017), exclusively relying on their public interfaces. We provide substantial experimental evidence to showcase the effectiveness of the Bregman algorithm framework. While our primary focus is on NPC and fairness ML, our framework holds significant potential for a broader range of constrained learning applications. The source code is currently freely available at~\href{https://github.com/zhenweilin/ConstrainedGBM}{https://github.com/zhenweilin/ConstrainedGBM}.
\end{abstract}

\section{Introduction}
The primary objective of conventional supervised learning is to generate a predictive model that captures the association between input and output by optimizing empirical errors. However, several challenges arise in training models when dealing with real-world applications.
For instance, in cancer diagnosis and fraud detection, since different classification errors produce varying loss, it is imperative to assign different weights to different categories to obtain a more realistic classification model, which is also referred to as NPC~\citep{rigollet2011neyman,tong2016survey}. 
Reports have surfaced about the risk of bias and discriminatory decisions affecting people based on race, gender, age, and other sensitive attributes. This awareness led to the rise of Fair ML~\citep{zafar2017fairness,agarwal2018reductions,cotter2019optimization}, a research area focused on measuring and mitigating the risks of bias in ML systems. 
These constrained learning problems require additional constraints to be taken into account during optimization, such as diversity in the recommendation system~\citep{castells2021novelty}, the misclassification rate of some classes in NPC,  fairness in Fair ML, monotonicity or convexity in shape-restricted regression~\citep{sen2017testing}, and smoothness constraint in semi-supervised learning~\citep{ouali2020overview}. 
Due to these constraints, conventional ML methods like logistic regression, Gradient Boosting Machine (GBM), and neural networks are no longer applicable.
To resolve this issue, we focus on developing a framework incorporating existing solvers to solve a wide range of constrained learning problems. More specifically, we focus on two important applications: NPC and Fair ML.

\textbf{NPC.} NPC is beyond the scope of standard machine learning solvers for unconstrained optimization.
Recently, NPC has attracted a growing research interest~\citep{cannon2002learning,scott2005neyman,tong2020neyman,mossman1999three,dreiseitl2000comparing,landgrebe2005neyman,tong2016survey}.
Due to the difficulty in directly minimizing classification
error rates, \citet{rigollet2011neyman,han2008analysis} proposed
replacing the 0-1 loss function with convex surrogate functions. \citet{tong2018neyman}
presented an umbrella algorithm adaptable to popular methods
such as logistic regression, SVM, and random forests. 
Aside from NPC, cost-sensitive learning~\citep{elkan2001foundations,zadrozny2003cost,li2020bridging} is another important paradigm to address the asymmetric error control, wherein practitioners manually assign specific weights to misclassification errors for each category.

\textbf{Fair Machine Learning.} Fair ML techniques are typically categorized into three groups: pre-processing, in-processing, and post-processing~\citep{mehrabi2021survey}.
Pre-processing methods strive to achieve an unbiased representation of the training data, but they may not guarantee fairness in the classifier~\citep{edwards2015censoring}. In contrast, post-processing methods require access to sensitive attributes after training and may be sub-optimal~\citep{woodworth2017learning}. 
In-processing Fair ML is an approach that modifies the learning procedure~\citep{zafar2017fairness,agarwal2018reductions,cotter2019optimization}.
Recently, \citet{cruz2022fairgbm} proposed the FairGBM framework that applies gradient boosted decision trees to the fairness classification problem. Nonetheless, FairGBM adopts a straightforward  gradient-descent-ascent approach for the associated min-max (nonconvex-concave) problem. Unfortunately, this approach lacks theoretical convergence guarantees, leaving room for potential improvement or further investigation.

To build a theoretically justified framework for constrained learning, we closely investigate the optimization methods applied to constrained learning.
By utilizing convex surrogate loss functions, NPC can be formulated as a convex problem and Fair ML  as a nonconvex problem with DC constraints.
There is a vast literature on convex function constrained optimization~\citep{yang2017richer,xu2021firstordera,bayandina2018mirror,lan2016iterationcomplexity,lin2018levelset,lan2013iterationcomplexitya}.
The iterative format of the Accelerated Primal-dual algorithm (APD)~\citep{hamedani2021primal} is relatively simple, comprising just three steps: gradient extrapolation, dual ascent using an explicit expression, and proximal gradient descent for the primal variable. 
For nonconvex constrained optimization, penalty methods~\citep{cartis2014complexity,wang2017penalty,cartis2011evaluation} and inexact proximal point method~\citep{boob2023stochastic,ma2020quadratically,lin2019inexact,boob2022level} are popular approaches. 
Note that all the mentioned  works on nonconvex constrained optimization require certain constraint qualification (CQ) to ensure convergence to the stationary conditions. This requirement was recently circumvented in a study by~\citep{jia2022first}, which demonstrated that the inexact proximal point method can still guarantee convergence to the Fritz John condition even when CQ fails.

In view of those recent advancements in constrained optimization,
our goal is to develop new  algorithm frameworks for constrained learning. We place particular emphasis on NPC and fairness ML. Furthermore, considering the proliferation of open-source software in machine learning, it is highly desirable that the new algorithmic framework can seamlessly integrate with standard ML packages, such as Sklearn and GBM, serving as the foundational tools.

\textbf{Contributions.} To solve the constrained learning, we propose two Bregman proximal point algorithms, Accelerated Bregman primal-dual Proximal Point algorithm ({\apd}) for convex constrained learning and Constraint Proximal Point Regularized algorithm ({\ilcp}) for nonconvex constrained learning. Our contributions are summarized as follows:

1. We provide {\apd} for solving constrained convex learning problems. In the iterative process, we adapt the objective function of an existing classifier to address the inherent sub-problems effectively. A key benefit of {\apd} is its seamless compatibility with prevalent classification models like GBM. Furthermore, we provide conclusive convergence proofs for {\apd}. To broaden compatibility across classifiers, we substitute the conventional $\ell_2$ norm with the Bregman distance as the penalty term for sub-problem resolution.

2. We propose {\ilcp} for non-convex learning problems. {\ilcp} solves the original problem iteratively by calling {\apd} to solve a sequence of convex subproblems. Therefore, {\ilcp} also works seamlessly with existing manual classifiers. Additionally, we have established that {\ilcp} is capable of converging to the FJ point, which is in line with the findings of~\citep{jia2022first}.

3. We show that the gradient boosting tree model is applicable to our proposed algorithms and give an efficient implementation of XGBoost~\citep{chen2016xgboost} and LightGBM~\citep{ke2017LightGBM} for solving NPC and fairness classification. Furthermore, we have extensively tested our algorithms on numerous publicly accessible datasets to confirm their efficacy.

\section{\label{sec:Background-and-Preliminary}Background and preliminaries}

Throughout the paper, we use bold letters to represent vectors and
matrices. Notation $[m]$ is short
for the set $\{1,\ldots,m\}$. An indicator function $\onebf(\cdot)$
takes on a value of $1$ if its argument is true and 0 otherwise.
For any vector $\xbf$, we define $[\xbf]_{+}$ as element-wise application
of $\max\{0,x_{(i)}\}$, where $x_{(i)}$ is the $i$-th coordinate
of $\xbf$. Let $\{\Abf,\bbf\}=\{\abf_{i},b_{i}\}_{i=1}^{n}$ be the
training set. We denote $F(\abf_i;\xbf)$ as the mapping of the feature to the classification score, where $\xbf$ is the classifier parameter. Denote the maximum and minimum singular values of matrix $\Dbf$ as $\sigma_{\max}(\Wbf)$ and $\sigma_{\min}(\Wbf)$, respectively. We define the domain of function $f$ as  $\operatorname{dom}f:=\{\xbf\mid f(\xbf)<\infty\}$. 
We assume that the loss function is convex with respect to the classifier parameter $\xbf$, which is valid for several common classifiers, including logistic regression, SVM and GBM~\citep{friedman2001greedy}.
Next, we give some basic definitions.
\begin{defn}[$\mu$-$\mathbf{D}_{\mathbf{W}}$ relative strong convexity]\label{defn:rel_strong_convex}
A function $f(\xbf)$ is $\mu$-$\mathbf{D}_{\mathbf{W}}$ relatively strongly convex in $\operatorname{dom}f$ if there exists a constant $\mu>0$ such that
\begin{align*}
f(t\protect\xbf+(1-t)\protect\xbf^{+}) \protect\leq tf(\protect\xbf)+(1-t)f(\protect\xbf^{+})\protect
-\mu t(1-t) \mathbf{D}_{\mathbf{W}}(\protect\xbf,\protect\xbf^{+}),\forall\protect\xbf,\protect\xbf^{+}\in \operatorname{dom} f, \label{}\vspace{-0.5em}
\end{align*}
where $\mathbf{D}_{\mathbf{W}}(\xbf,\xbf^{+})=\tfrac{1}{2}\left\langle \xbf-\xbf^{+},\mathbf{W}^{\top}\mathbf{W}(\xbf-\xbf^{+})\right\rangle $
is a Bregman distance induced by a matrix $\mathbf{W}$.
\end{defn}
\begin{defn}[$L_g$-$\mathbf{D}_{\mathbf{W}}$ relative Lipschitz continuity]\label{defn:rel_lip_con}
We say that a map $\gbf:\mathbb{R}^d\to\mathbb{R}^m$ is $L_g$-$\mathbf{D}_{\mathbf{W}}$ relatively Lipschitz continuous in $\operatorname{dom} f$ if there
exists $L_{g}>0$ such that $\|\mathbf{g}(\xbf)-\mathbf{g}(\xbf^{+})\|\leq L_{g}\sqrt{2\mathbf{D}_{\mathbf{W}}(\xbf,\xbf^{+})}$
holds for $\forall\mathbf{x},\mathbf{x}^{+}\in \dom f$.
\end{defn}

\begin{defn}[$L_h$-$\mathbf{D}_{\mathbf{W}}$ relative smoothness]\label{defn:relative_smooth}
    We say that a function $h:\mathbb{R}^d\to \mathbb{R}$ is $L_h$-$\mathbf{D}_{\mathbf{W}}$ relatively smooth if there exists $L_h>0$ such that $|h(\xbf^+)$ $-\,h(\xbf) -  (\nabla h(\xbf))^{\top}(\xbf^+   - \xbf)| \leq   L_h\mathbf{D}_{\mathbf{W}}(\xbf^+,\xbf)$ holds for any $\xbf,\xbf^+ \in   \dom f.$
\end{defn}

The concepts of strong convexity, Lipschitz continuity, and Lipschitz smoothness in the relative scale have been explored in several recent works~\citep{lu2019relative,zhou2020regret}. These studies extended the standard definitions in convex analysis,  which are grounded in the Euclidean metric (where $\Wbf=\Ibf$), by adapting them to accommodate more general distance metrics. A notable example can be found in~\citet{dhillon2008matrix}. By setting $\Wbf^{\top}\Wbf$ as the inverse covariance matrix of multivariate Gaussian distributions, the divergence serves as a measure of the distance between points generated by the Gaussian distribution. Moreover, in the context of our proposed GBM-based algorithms, we set $\Wbf$ as a scoring matrix of the training data determined by individual weak learners. This transformation enables the conversion of the standard Euclidean space-based distance measure into a more tractable, performance-based measure, which has proven particularly effective for ensemble models.

\section{\label{sec:Algorithm-and-Convergence}
The Accelerated Bregman Primal-dual Proximal Point Method
}
We propose the {\apd} method and develop its convergence analysis for optimizing convex-constrained models. A prominent instance of learning under convex constraints is given by NPC. For multi-classification, NPC employing logistic loss can be expressed as:
\begin{gather}
\begin{aligned}\vspace{-1.5em}
\raisetag{2\baselineskip}
\min_{\xbf\in \mcal{X}} &~ \ell_{0}(\xbf)=\sum_{i=1}^{n} w_i \ell (F(\abf_{i};\xbf),b_{i})\\
\st&~ \ell_{j}(\xbf)=\sum_{i=1}^{n}\ell(F(\abf_{i};\xbf),b_{i})\onebf(b_{i}=j)\leq\alpha_{j}, j\in[J],
\end{aligned}
\label{eq:NPMC}\vspace{-0.5em}
\end{gather}
where $w_i$ is the weight for the $i$-th data point, $\ell(\hat{\ybf}_i,b)=-b_{(j)}\log \brbra{\exp(\hat{\ybf}_{i(j)})/{\sum_{j=1}^{J}\exp(\hat{\ybf}_{i(j)})}}$, $\hat{\ybf}_i = F(\abf_i;\xbf)$, $b_{(j)} = 1$ means the data point is labeled as $j$, $J$ is the total number of classes, $\alpha_j$ is a predefined NPC violation level and $\mcal X\subseteq \dom \ell_0$ is a closed convex set.
Without loss of generality, we develop the algorithm  and its convergence analysis all based on the following general formulation:
\begin{equation}
\min_{\xbf\in\mcal X}~f(\xbf)\ \ \st~ g_{i}(\xbf)\leq0,i=1,2,\ldots,m,
\label{eq:np_prob_general}
\vspace{-0.2em}
\end{equation}
where $f(\xbf)$ and each $g_{i}(\xbf)$ are convex continuous functions. Moreover, we write $\mathbf{g}(\xbf)=\sbra{g_{1}(\xbf),\cdots,g_{m}(\xbf)}^{\top}$.

Algorithm~\ref{alg:APDinexact} describes the {\apd} method
for solving problem~(\ref{eq:np_prob_general}). Within
each iteration, {\apd} performs a dual update by leveraging the extrapolated
gradient, denoted as $\gbf(\xbf)$.
For the primal variables, {\apd} requires an external solver to approximately solve the proximal sub-problem.

\paragraph{Why the proximal point method?} 
The key distinction between Algorithm~\ref{alg:APDinexact} and the gradient-based algorithm lies in line~\ref{lst:line:minphi_inexact} of {\apd} for updating the primal variable:
Instead of performing a proximal gradient step, {\apd} solves a proximal point problem associated with the Lagrangian function (i.e., $f(\xbf)+\ybf^\top \gbf(\xbf)$) and the Bregman divergence. The reason can be summarized as two folds. One is that, theoretically, the proximal point method can handle the non-smooth problem, thus improving the applicability of the algorithm. However, when the constraint function $\gbf$ or the objective function $f$ is non-smooth, it will not be possible to take the proximal gradient step. The other is that in practice, the proximal point method allows us to leverage the off-the-shelf machine learning libraries with minimal adjustments. However, the gradient method requires careful examination of the classifier framework's source code to ensure efficient implementation, a process that incurs lengthy development time and cannot be readily extended to highly integrated and mature classifiers.

We develop the complexity of Algorithm~\ref{alg:APDinexact}.  To this end, we need to  define some optimality criteria:

\begin{defn}\label{def:approx_stochastic}
    We say $\xbf$ is a $(\delta,\nu)$-approximate solution of $\psi$ at $\qbf$ if 
    \begin{equation}
        \psi(\xbf)-\psi(\xbf^*)\leq \delta \Dbf_{\Wbf}(\xbf,\qbf)+\nu,
    \end{equation}
    where $\xbf^*$ is the exact solution of $\min_{\xbf\in \mcal X} \psi(\xbf)$.
\end{defn}

The above definition is highly flexible and can be adapted
to many popular training algorithms. For example, when $f$ and $g_i$
are smooth, and the sub-problem is solved by gradient-based solvers, we will
have $\nu=0$. When $f$ and $g_i$ are nonsmooth, often the error is dominated by the
$\Ocal(\nu)$ term. In view of Definition~\ref{def:approx_stochastic},
we give {\apd} in Algorithm~\ref{alg:APDinexact}.
\vspace{-0.5em}
\begin{algorithm}[htpb]
\caption{\label{alg:APDinexact}The Accelerated Bregman primal-dual Proximal Point method ({\apd})}
    \begin{algorithmic}[1]
        \REQUIRE{$\tau_0>0, \sigma_{-1}=\sigma_0>0$, $\xbf_{-1}=\mathbf{x}_0,\ybf_{-1}=\mathbf{y}_0,\mu$}
        \STATE{\textbf{Initialize:} $\gamma_0=\sigma_0/\tau_0$,$k=0$}
        \WHILE{$k<K$}
		\STATE{$\sigma_k \leftarrow \gamma_k \tau_k$,$\theta_k \leftarrow \frac{\sigma_{k-1}}{\sigma_k}$}
        \STATE{$\mathbf{z}_{k}\leftarrow (1+\theta_k)\mathbf{g}(\mathbf{x}_k) -  \theta_k \mathbf{g}(\mathbf{x}_{k-1})$}
\STATE{
$\mathbf{y}_{k+1}\leftarrow [\ybf_k + \sigma_k \zbf_k]_+$
}
\STATE{\label{lst:line:minphi_inexact}Find $\xbf_{k+1}$, a $(\delta_{k+1},\nu_{k+1})$-approximate solution of $\psi_k(\xbf):=f(\xbf)+\inprod{\ybf_{k+1}}{\gbf(\xbf)}+\tau_k^{-1}\Dbf_{\Wbf}(\xbf,\xbf_{k})$ at $\xbf_{k}$ .}
\STATE{$\gamma_{k+1}\leftarrow \gamma_k(1+\mu \tau_k)$,$\tau_{k+1}\leftarrow \tau_k \sqrt{\frac{\gamma_k}{\gamma_{k+1}}}$,$k\leftarrow k + 1$}
        \ENDWHILE
    \STATE{\textbf{Output: $\mathbf{x}_{K}$}}
        \end{algorithmic}
\end{algorithm}

Let $f^{*}=f(\xbf^{*})$
where $\xbf^{*}$ is an optimal solution of problem~\eqref{eq:np_prob_general}.
We say that a point $\xbf$ is an \textit{$\vep$-optimal solution}
if $f(\xbf)-f^{*}\leq\vep$ and $\|[\gbf(\xbf)]_{+}\|\leq\vep.$
Let $\bar{\xbf}_{K}:=\frac{1}{T_{K}}\sum_{k=0}^{K-1}t_{k}\xbf_{k+1},\bar{\ybf}_{K}=\frac{1}{T_{K}}\sum_{k=0}^{K-1}t_{k}\ybf_{k+1}$
be the weighted average of the ergodic sequence $\{\xbf_{k},\ybf_{k}\}_{k=1}^{K}$,
where $T_{K}=\sum_{k=0}^{K-1}t_{k}$. We present the main convergence property as follows.
Due to space limits, we leave all the proof detail in the appendix.

\begin{thm}
\label{thm:inexact_thm_expected}
Suppose $f$ is $\mu$-$\mathbf{D}_{\mathbf{W}}$ relatively strongly convex in $\mcal X$, $\gbf$ is $L_g$-$\mathbf{D}_{\mathbf{W}}$ relatively Lipschitz continuous in $\mcal X$, 
and there exists a constant $\delta>0$ such that the sequence $\{\tau_{k},\sigma_{k},t_{k}\}_{k\geq0}$
satisfies 
\begin{equation*}
\begin{split}t_{k+1}\tau_{k+1}^{-1}\leq t_{k}(\tau_{k}^{-1}+\mu),& t_{k+1}\theta_{k+1}=t_{k}, \theta_{k}\sigma_{k-1}=\sigma_{k}, \\
L_{g}\sigma_{k}/\delta\leq(\tau_{k})^{-1} ,\theta_{k}\delta\sigma_{k}& \leq\sigma_{k-1},t_{k+1}\sigma_{k+1}^{-1}\leq t_{k}\sigma_{k}^{-1}.
\end{split}
\end{equation*}
Then,  for the sequence $\{\xbf_{k},\ybf_{k},\bar{\xbf}_{k},\bar{\ybf}_{k}\},$
generated by Algorithm~\ref{alg:APDinexact}, we have
\begin{equation*}
\protect\begin{split}T_{K}[\protect\mcal L(\bar{\protect\xbf}_{K,}\protect\ybf)-\protect\mcal L(\protect\xbf,\bar{\protect\ybf}_{K})+t_{K-1}\tau_{K-1}^{-1}\mathbf{D}_{\mathbf{W}}(\protect\xbf,\protect\xbf_{K})]\protect
\protect\leq\tfrac{t_{0}\mathbf{D}_{\mathbf{W}}}{\tau_{0}}(\protect\xbf,\protect\xbf_{0})+\tfrac{t_{0}\sigma_{0}^{-1}}{2}\protect\norm{\protect\ybf-\protect\ybf_{0}}^{2}+\protect\mbb{\sum}_{k=0}^{K-1}t_{k}\eta_{k},
\protect\end{split}
\end{equation*}
where $\eta_{k}:=2\sqrt{\delta_{k+1}\mathbf{D}_{\mathbf{W}}(\hat{\xbf}_{k+1},\xbf_{k})+\nu_{k+1}}\allowbreak\sqrt{\tfrac{1}{\tau_{k}}+\mu}\allowbreak\sqrt{\mathbf{D}_{\mathbf{W}}(\xbf,\xbf_{k+1})}$.

\end{thm}

We develop the convergence rate results for more specific parameter settings.

\begin{cor}[Informal]\label{cor:inexact_cor_expected}
    Suppose assumptions of Theorem~\ref{thm:inexact_thm_expected} hold.
    Set sequence $t_k=\sigma_k/\sigma_0$ and $\tau_0\sigma_0\leq \delta/L_g$. Then we have $\max \{f(\bar{\xbf}_k)-f(\xbf^*),\norm{[\gbf(\bar{\xbf}_{K})]_+}\}=\mcal O(T_K^{-1})$. Furthermore, take $\delta_{k+1}=\nu_{k+1}=\Ocal(k^{-7})$ when $\mu>0$, then we get an $\vep$-optimal solution at $\Ocal(1/\sqrt{\vep})$ iterations. Take $\delta_{k+1}=\nu_{k+1}=\Ocal(k^{-4})$ when $\mu=0$, then we get an $\vep$-optimal solution at $\Ocal (1/\vep)$ iterations.
\end{cor}

\begin{rem}
\label{rem:overall_iteration}Corollary~\ref{cor:inexact_cor_expected} illustrates that \apd{} requires $\mathcal{O}(\vep^{-1/2})$ iterations
to attain an $\vep$-optimal solution when $\mu>0$, and $\mathcal{O}(\vep^{-1})$ when $\mu=0$. Note that he total complexity should also incorporate the time complexity associated with solving each subproblem. 
Consider the case when  $\Wbf$
is an identity matrix, and both $f(\xbf)$ and $g_{i}(\xbf)$ are smooth convex, employing variance-reduced methods such as SAGA~\citep{defazio2014saga}
and SVRG~\citep{xiao2014proximal} can yield a linear convergence rate. This introduces an additional 
$\mathcal{O}(\log \vep^{-1})$ term in computing the total complexity.
\end{rem}

\begin{rem}
    In Section~\ref{sec:gbm_abpd}, we show that when $\Wbf$ is the score matrix of weak learners, the sub-problem can be solved by using GBM.
    Furthermore, there exists a substantial body of literature on the convergence analysis of GBM for subproblems. Some of the well-known works in this field include~\citep{mukherjee2013rate,bickel2006some,telgarsky2012primal,freund2017new}. \citep{lu2020accelerating} proposes a novel algorithm named Accelerated Gradient Boosting Machine with Restart~(AGBMR), which can achieve $\mathcal{O}(\log({\vep^{-1}}))$. Additionally, \citet{lu2020randomized} conduct a convergence analysis of the Randomized GBM~(RGBM) and obtains a similarly impressive linear convergence rate.
\end{rem}

\section{\label{sec:alg_fair} The Constrained Bregman Proxiaml Regularized Method}

This section presents the {\ilcp} algorithm and its convergence analysis for non-convex constrained optimization. A notable application is fairness classification, which can be formulated as follows: 
$\min_{\xbf\in\mcal X} \sum_{i=1}^{n}\ell(F(\abf_{i};\xbf),b_{i})~\st  |\xi(j)-\xi(l)|\leq\alpha_{jl},~ j < l,\forall j, l\in [|\mcal S|]$
where $\xi(k):=\tfrac{1}{n_{k}}\sum_{i=1}^{n}\ell(F(\abf_{i};\xbf), b_i \mid s_i=k)$, $\abs{\mcal{S}}$ is the cardinality of the sensitive attribute set $\mcal{S}$,  $n_{j}=\sum_{i=1}^{n}\onebf(s_{i}=j)$ is the number
of samples in the $j$-th sensitive attribute and $\alpha_{jl}$ is the pre-specified fairness violation level. The above objective 
is a general loss function such as logistic
loss  and the constraints to restrict ``average loss difference'' is
less than the threshold $\alpha_{sl}$ among sensitive attribute datasets.
Without loss
of generality, we give a compact formulation as follows:
\begin{gather} \raisetag{2\baselineskip}
\begin{split}\min_{\xbf\in\mcal X} &~ \phi_{0}(\xbf):=f(\xbf)\\
\st &~\bar{\phi}(\xbf):= \max_{i\in[m]}\{g_{i}(\xbf)-h_{i}(\xbf)-\eta_{i}\}\leq0,
\end{split}
\label{eq:fair_prob}
\end{gather}
where $f(\xbf):\mcal X\to\mbb R\cup\{+\infty\}$ is a convex lower-semicontinuous
function defined in the compact set $\mcal X$, $g_{i}(\xbf),h_{i}(\xbf)$
are proper convex functions. $h_{i}(\xbf)$ is $L_{h_{i}}-\mathbf{D}_{\mathbf{W}}$
relatively Lipschitz smooth. Next, we give the algorithm
description and convergence based on~\eqref{eq:fair_prob}. First,
we give Inexact Constrained Bregman Proximal method ({\ilcp}) in Algorithm~\ref{alg:LCP}. Given that $h_i(\xbf)$ exhibits relative smoothness, we can obtain a relatively strongly convex constraint by adding a Bregman proximal term to the constraint function, where the modulus of the proximal term is  greater than the smooth coefficients of $h_i(\xbf)$ in $\bar{\phi}(\xbf)$.
Within each iteration, {\ilcp} solves a relatively strongly convex functional constrained sub-problem approximately (see Definition~\ref{defn:optimal} for the approximation criteria).
The sub-problem, which is constructed by adding a large enough proximal term depending on the last iteration to the original problem~\eqref{eq:fair_prob},
 can be solved by {\apd} proposed in Section~\ref{sec:Algorithm-and-Convergence}. 
\begin{algorithm}
\caption{\label{alg:LCP} Constrained Bregman Proximal Regularized method ({\ilcp})}
\begin{algorithmic}[1]
\REQUIRE{A feasible point $\xbf^0$ for problem~\eqref{eq:fair_prob}, $L>\max\bcbra{\max_{i}\bcbra{L_{h_{i}}},\sigma_{\max}(\mathbf{W}^{\top} \mathbf{W})^{-1}}$}
        \FOR{$t=0,1,\ldots, T-1$ } 
        \STATE{Call {\apd} to find an $(\vep_3,\vep_4)$ optimal solution $\xbf^{t+1}$ (Definition~\ref{defn:optimal}) for the following problem
\begin{equation}\label{eq:lcp_subprob} \begin{split}\min_{\xbf\in\mcal X} &~~ \phi_t^0(\xbf):=f(\xbf) + L\mathbf{D}_{\mathbf{W}}(\xbf,\xbf^{t})\\ \st &~~ \bar{\phi}_t(\xbf):=\bar{\phi}(\xbf)+L\mathbf{D}_{\mathbf{W}}(\xbf,\xbf^{t})\leq 0\end{split} \end{equation}
}
\ENDFOR
\STATE{\textbf{Output:} $\xbf^{T}$}
\end{algorithmic}
\end{algorithm}

\begin{defn}[$(\vep_3,\vep_4)$-optimal solution]\label{defn:optimal}
We say $\xbf$ is an $(\vep_{3},\vep_{4})$
optimal solution for problem~\eqref{eq:lcp_subprob} if $\phi_{t}^{0}(\xbf)-\phi_{t}^{0}(\xbf^{*})\leq\vep_{3}$
and $\bar{\phi}_{t}(\xbf)\leq\vep_{4}$, where $\xbf^{*}$ is an
optimal solution.
\end{defn}

Before providing the main convergence analysis, we 
introduce the Fritz-John (FJ) condition.
\begin{defn}[FJ point]\label{def:optimal_fj}
 We say a feasible $\xbf^{*}$ is a FJ point of~\eqref{eq:fair_prob}
if there exists subgradient $d_{\phi0}\in\partial\phi_{0}(\xbf^*)$, $d_{\phi_i}\in\partial (g_i(\xbf^*) - h_i(\xbf^*))$ and nonnegative multipliers $y_{0}^{*}\in\mbb R_{+}$
and $\ybf^{*} = (y_1^*, y_2^*,\cdots, y_m^*)^{\top}\in\mbb R_{+}^m$ such that
\begin{align}
y_i^{*} (g_i(\xbf^*)-h_i(\xbf^*)-\eta_i) & =0,\forall i \in [m]\label{eq:fj_point_1}\\
y_{0}^{*}d_{\phi_{0}}+ \sum_{i=1}^{m} y_i^{*}d_{\phi_i} & \in-\mcal N_{\mcal X}(\xbf^{*}).\label{eq:fj_point_2}
\end{align}
\end{defn}

In particular, \citet{jia2022first} gives a unified algorithm that guarantees convergence to the KKT condition when the CQ holds, and converges to the FJ condition when CQ fails. 

Next, we propose the following approximate FJ condition to measure the algorithm performance.
\begin{defn}
[Approximate FJ point]\label{def:fjpoint} We say a point $\xbf$ is an
$\vep$-FJ point for problem~\eqref{eq:fair_prob} if $\bar{\phi}(\xbf)\leq0$
and there exists $d_{\phi0}\in\partial\phi_{0}(\xbf)$ and $d_{\phi_i}\in\partial (g_i(\xbf) - h_i(\xbf))$
and $y_{0}\geq0, \ybf = [y_1, y_2,\cdots, y_m]^{\top}\in\mbb R_{+}^m$, $\sum_{i=0}^{m} y_{i}=1$ such that 
\[\dist(y_{0}d_{\phi0}+\sum_{i=1}^{m}y_i d_{\phi_i},-\mcal N_{\mcal X}(\xbf)),\sqrt{\abs{y_i d_{\phi_i}}} \leq\vep,\forall i \in [m].\]
 We say a point $\xbf$ is an $(\vep_{1},\vep_{2})$-FJ
point for problem~\eqref{eq:fair_prob} if there exists an $\vep_{1}$-FJ
point $\hat{\xbf}$ for problem~\eqref{eq:fair_prob} with $\norm{\xbf-\hat{\xbf}}\leq\vep_{2}$.
\end{defn}

We state the main convergence property of {\ilcp} in the following theorem.  Proof details
are deferred to the Appendix. 
\begin{thm}
\label{thm:3.2}
Suppose $\inf_{\xbf\in\mcal X}f(\xbf)>-\infty$ and taking $\vep_{3}=\vep_{4}=\tfrac{(L-\rho)\vep^{2}\sigma_{\min}(\mathbf{W}^{\top} \mathbf{W})}{4L^{2}\sigma_{\max}(\mathbf{W}^{\top} \mathbf{W})^{2}}$, where $\rho = \max_{i\in [m]}\bcbra{L_{h_i}}$. Then, {\ilcp} return an $(\vep, \vep)$-FJ point within at most $
T:=\tfrac{4L^{2}\sigma_{\max}(\mathbf{W}^{\top} \mathbf{W})^{2}(f(\xbf^{0})-\inf_{\xbf\in\mcal X}f(\xbf))}{3(L-\rho)\sigma_{\min}(\mathbf{W}^{\top} \mathbf{W})\vep^{2}}-1$
outer iterations, and the overall iteration number is $\mcal{O}(1/\vep^3)$ when using {\apd} to solve~\eqref{eq:lcp_subprob}. 
\end{thm}

The FJ condition is more general than the Karush-Kuhn-Tucker (KKT) condition. Specifically,  A KKT point is a special case of FJ point when the dual variable $y_0^*\neq 0$.
To validate the KKT condition, an additional assumption known as constraint qualification is needed.  See \citep{ma2020quadratically, boob2022level,jia2022first}.

\section{\label{sec:ABPD_NPC}{\apd} and {\ilcp} for Supervised Learning}

In this section, we employ the {\apd} and {\ilcp} frameworks to train classification models for the NPC and fairness ML problems, respectively. Our approach involves using off-the-shelf classifiers (such as linear models and GBM) to solve the sub-problems in {\apd} and {\ilcp}. First, recall the convex relaxation of the NPC formulation in~\eqref{eq:NPMC} (binary classification is a special case) and the fairness formulation in~\eqref{eq:fair_prob}. For notation simplicity, let $\bell(\xbf)=[\ell_{1}(\xbf),\cdots,\ell_{m}(\xbf)]^{\top},\bm{\alpha}=[\alpha_{1},\cdots,\alpha_{m}]^{\top}$.

\paragraph{Linear Classifier-based {\apd}\label{abpd_linear}}
To effectively employ {\apd} in training a linear NPC model represented as $F(\abf;\xbf)=\abf^{\top}\xbf$, it's crucial to ascertain  that the assumptions of Theorem~\ref{thm:inexact_thm_expected} are satisfied, i.e.,  $\ell_0(\xbf)$  and $\bell(\xbf)$ are $\mu$-$\Dbf_{\Wbf}$ relatively strongly convex and  $L_{\bell}$-$\Dbf_{\Wbf}$ relative Lipschitz continuous w.r.t. $\xbf$, respectively.  Suppose we use the cross-entropy loss as a convex surrogate and $\Wbf = \Ibf$, then we have $L_{\bell} = \sum_{i=1}^{n}\norm{\abf_i}$ and $\mu\geq 0$ ($\mu>0$ for common loss with 2-norm regularizer).  
With this setup, the sub-problem in {\apd} can be written as $\min_{\xbf\in\mathcal{X}}\ell_{0}(\xbf)+\left\langle \ybf_{k+1},\bell(\xbf)\right\rangle +\frac{1}{2\tau_{k}}\|\xbf-\xbf_{k}\|^{2}$,
implying that the objective with respect to  $\xbf$ is strongly convex. Given the maturity of machine learning libraries for regularized learning, implementation becomes straightforward. For instance, an efficient adaptation can be achieved by making minor modifications to the objective function within the Scikit-learn library~\citep{scikit-learn}.

\paragraph{GBM-based {\apd}\label{sec:gbm_abpd}}
When implementing {\apd} with GBM, it is crucial to ensure the relative convexity with appropriate choices of  $\mu$, $\Wbf$, and $L_{\ell}$.
Assuming that there are  $N$ weak learners, the classic GBM can be described as the following additive model:
$
F(\abf;\xbf):=\sum_{i=1}^{N}x_{(i)}f_{i}(\abf;\iota_{i})=\tilde{F}(\abf)\xbf,
$
where $f_{i}(\abf;\iota_{i})$, mapping features to scores in each class, is a weak learner parameterized by
$\iota_{i}$, and $\xbf\in\mbb{R}^N$ represents the weight vector of the $N$ weak learners. In practice, commonly used weak learners include wavelet functions, tree stumps, and classification and regression trees. 
For brevity, we denote the vector $\tilde{F}(\abf)=[f_{1}(\abf;\iota_1),\cdots,f_{N}(\abf;\iota_N)]\in\Rbb^{J\times N}$
as the prediction scores of all the weak learners. 
 Let $F(\Abf;\xbf)=\tilde{F}(\Abf)\xbf$ be the score matrix of dataset $\Abf$,
where $\tilde{F}(\Abf)=[\tilde{F}(\abf_{1})^{\top},\cdots,\tilde{F}(\abf_{n})^{\top}]^{\top} \in \mbb{R}^{n J\times N}$.
Consider the cross-entropy loss as an example, the loss is convex with respect to the parameter $\xbf$ of GBM. Moreover, $\ell_0$ is a $0$-$\Dbf_{\tilde{F}(\Abf)}$ relative strongly convex function, and $\bell$ is $1$-$\Dbf_{\tilde{F}(\Abf)}$ relatively Lipschitz continuous.

\begin{rem}
Our algorithm framework offers versatility beyond the cross-entropy loss function. It can readily accommodate several other common loss functions used in GBM, such as Huber loss with $l_2$ penalty  and the least squares loss. The constants $\mu$, $\Wbf$, and $L_{\bell}$ for these examples can be calcluated accordingly.
\end{rem}

Using the proximal term $\mathbf{D}_{\tilde{F}(\Abf)}$, we write the
sub-problem of GBM-based ABPD as $\min_{\xbf} \allowbreak \psi_{k}(\xbf)\allowbreak :=\ell_{0}(\xbf)\allowbreak+\left\langle \ybf_{k+1},\bell(\xbf)-\boldsymbol{\alpha}\right\rangle +\tau_{k}^{-1}\mathbf{D}_{\tilde{F}(\Abf)}(\xbf,\xbf_{k}),$
where $\xbf_{k}$ is the approximate solution at $(k-1)$-th iteration.
Indeed, $\psi_{k}(\xbf)$ is a weighted sum of different classification losses with Bregman distance penalty
with projection operator $\tilde{F}(\Abf)$, which can be optimized
by widely used GBM, such as AdaBoost~\citep{mukherjee2013rate}, XGBoost~\citep{chen2016xgboost}
and LightGBM~\citep{ke2017LightGBM}. In practice, by customizing the \footnote{https://github.com/dmlc/xgboost}{objective function of XGBoost} and \footnote{https://github.com/microsoft/LightGBM/tree/master/src/objective}{the objective function of LightGBM}, we can easily solve the sub-problem $\psi_{k}(\xbf)$.

\paragraph{\label{sec:class_ilcp}Linear Classifier-Based and GBM-Based {\ilcp}}
The primary concept behind {\ilcp} involves converting the initial non-convex problem into a sequence of relatively strongly convex sub-problems that can be tackled using {\apd}. Therefore, it is natural to embed the linear classifier and GBM into the {\ilcp} framework for solving fairness classification, which will not be repeated here.

\section{\label{sec:Experiments}Numerical Study}
\begin{table}[]
  \centering
    \begin{tabular}{clcccc}
    \toprule
    \multicolumn{1}{l}{Dataset} & Model & \multicolumn{2}{c}{Accuracy/std(\%)} & \multicolumn{2}{c}{Con Vio/std(\%)} \\
    \midrule
    \multirow{4}[4]{*}{Credit} & LGBM(B) & \textbf{99.92} & \textbf{0.04}  & 22.94 & 13.67 \\
          & LGBMNP(O) & 96.77 & 2.72  & \textbf{6.63}  & \textbf{3.18} \\
\cline{2-6}          & XGB(B) & \textbf{99.96} & \textbf{0.01}  & 15.69 & 4.56 \\
          & XGBNP(O) & 99.86 & 0.16  & \textbf{10.01} & \textbf{2.60} \\
    \hline
    \multirow{6}[6]{*}{Drybean} & RandF+CX & \textbf{52.01} & \textbf{3.16}  & 2.66  & 1.64 \\
          & RandF+ER & 49.12 & 4.30  & \textbf{1.27}  &\textbf{ 1.38} \\
\cline{2-6}          & LGBM(B) & \textbf{91.73} & \textbf{0.51}  & 12.87 & 2.22 \\
          & LGBMNP(O) & 86.39 & 1.95  &\textbf{ 3.23 } & \textbf{1.29} \\
\cline{2-6}          & XGB(B) & \textbf{91.79} &\textbf{ 0.65}  & 12.14 & \textbf{2.18} \\
          & XGBNP(O) & 90.74 & 1.44  &\textbf{ 9.32 } & 2.60 \\
    \bottomrule
    \end{tabular}%
    \vspace{-0.5em}
\caption{NPC Results in Test Set. (B: Baseline, O:Ours)}

  \label{tab:NPC}%
\end{table}
In this section, we evaluate our proposed approach through experiments on real-world datasets. We utilize {\apd} and {\ilcp} with convex surrogate loss functions to address NPC and fairness problems. Sub-problems within {\ilcp} are handled by directly invoking {\apd}. We modify XGBoost~\citep{chen2016xgboost} and LightGBM~\citep{ke2017LightGBM} classifiers to incorporate our method.
For comparison, we also evaluate RandF+CX (Convex with random forest) and RandF+ER (Empirical Risk with random forest)  from~\citep{tian2021neyman} for NPC. These models use the random forest and employ convex and empirical loss as surrogate losses for NPC.
We use the random forest as the benchmark to ensure fairness in comparison since it also uses the tree algorithm, similar to XGBoost and LightGBM.
We also compare our approach to FairGBM~\citep{cruz2022fairgbm}, a state-of-the-art fairness classifier compatible with LightGBM. Additionally, we include the vanilla model (general XGBoost or LightGBM without constraints) with the same model parameters as a control group.
We provide dataset descriptions in Table~\ref{tab:summary} and explain how to set $\alpha$ in NPC while introducing performance metrics for the algorithm.\vspace{-0.5em}

\begin{table}[htbp]
  \centering
    \begin{tabular}{l|rrrr|rrrr}
    \toprule
    Dataset & \multicolumn{4}{c|}{Compas}   & \multicolumn{4}{c}{Adult} \\
    Model & \multicolumn{2}{c}{Accuracy/std(\%)} & \multicolumn{2}{c|}{Fairness/std(\%)} & \multicolumn{2}{c}{Accuracy/std(\%)} & \multicolumn{2}{c}{Fairness/std(\%)} \\
    \midrule
    FairGBM & 62.67 & 1.81  & 14.74 & 6.04  & 77.79 & 9.19  & 17.20 & 10.17 \\
    LGBM(B) & 62.72 & 1.84  & 13.71 & 5.39  & 86.79 & 0.57  & 10.06 & 0.39 \\
    LGBMFair(O) & \textbf{62.96} & \textbf{1.27}  & \textbf{10.09} & \textbf{1.61}  & \textbf{87.22} & \textbf{0.21}  &\textbf{ 9.76}  & \textbf{0.15} \\
    \midrule
    XGB(B) & 63.05 & 0.78  & \textbf{26.54} & 1.43  & \textbf{86.96} & \textbf{0.41}  & \textbf{9.89}  & \textbf{0.34} \\
    XGBFair(O) & \textbf{63.46} & \textbf{0.33}  & 26.69 & \textbf{0.09}  & 84.54 & 0.95  & 10.90 & 0.71 \\
    \bottomrule
    \end{tabular}%
    \vspace{-0.5em}
    \caption{\label{tab:fair}Fairness Result in Test Set. (B: Baseline, O:Ours)}
\end{table}%

\paragraph{How to set $\alpha$ in NPC?} Initially, we calculate the loss value by assuming $1/K$ as the predicted probability for estimating $\alpha$. After a certain number of iterations, we estimate $\alpha$ based on the current overall predicted probability. The binary NPC problem is subject to the constraint $-\sum_{i=1}^{n}(1-y_i)\log(1-\hat{y}_i)\leq \alpha$, where $\hat{y}_i = F(\abf_i;\xbf)$. Hence, in the heuristic strategy, we first set $\alpha = n_0\cdot e \cdot \log (2)$, where $e$ is the expected error rate, $n_0$ is the dataset size of class $0$. After running half of the iterations, we modify $\alpha$ to $-e\sum_{i=1}^n (1-y_i)\log(1-\hat{y}_i)$. Similarly, in the case of multi-class NPC, we initially set $\alpha_k = n_k\cdot e_k\cdot \log K$, where $e_k$ is the expected error rate of class $k$ and $n_k$ is the dataset size of class $k$. After running half of the iterations, we adjust the $\alpha$ to $-e_k\sum_{i=1}^n \sum_{j=1}^{K} y_{ij}\log\hat{y}_{ij}$, where $\hat{y}_{ij} = \hat{\ybf}_{i(j)}$ and $\hat{\ybf}_{i} = F(\abf_i;\xbf)$. 

\begin{table}
\centering
\setlength{\tabcolsep}{1mm}
{\vspace{-0.5em}
\begin{tabular}{ccccc}
\toprule
{Dataset} & {SampleSize} & {Feature} & {Task} & {Source} \\\midrule
Credit                   & 284807                      & 30                       & NPC                   & {kaggle}                \\
Drybean                  & 13611                       & 15                       & NPC                   & {UCI}                   \\
Adult                    & 48842                       & 14                       & Fairness              & {UCI}                   \\
Compas                   & 6172                        & 6                        & Fairness              & {propublica}            \\ \bottomrule
\end{tabular}
}
\caption{\label{tab:summary}Dataset Description and Optimization Objective}
\end{table}

\paragraph{Evaluation metrics.}
We randomly select 100 sets of hyperparameters, including step size, iteration numbers, and other classifier hyperparameters, to compare different models' performance. Ensuring consistent parameters for the same classifier is crucial for reliable and comparable results.
For NPC, we evaluate constraint violation in the NPC problem using the term $\sum_{i=1}^{K}[\Pr_{\abf\mid b = i}(\phi(\abf)\neq i) - \ebf]_+$, where $\phi(\abf)$ maps data features to predicted labels. We plot Figure~\ref{fig:NPC} with this term as the horizontal axis and classification accuracy as the vertical axis.
To address fairness concerns, we calculate the misclassification rate for various sensitive attributes. We then plot the difference for datasets with two sensitive attributes or the standard deviation for datasets with multiple sensitive attributes on the horizontal axis. On the vertical axis, we represent the classification accuracy. This approach allows us to evaluate the model's fairness across different sensitive attributes and datasets. The fairness results are shown in Figure~\ref{fig:fair}. Subsequently, we will delve into a comprehensive analysis of the algorithm's effectiveness across four distinct tasks.

\begin{figure}[htp]
\begin{centering}
\begin{minipage}[t]{0.25\columnwidth}%
\includegraphics[width=4.25cm]{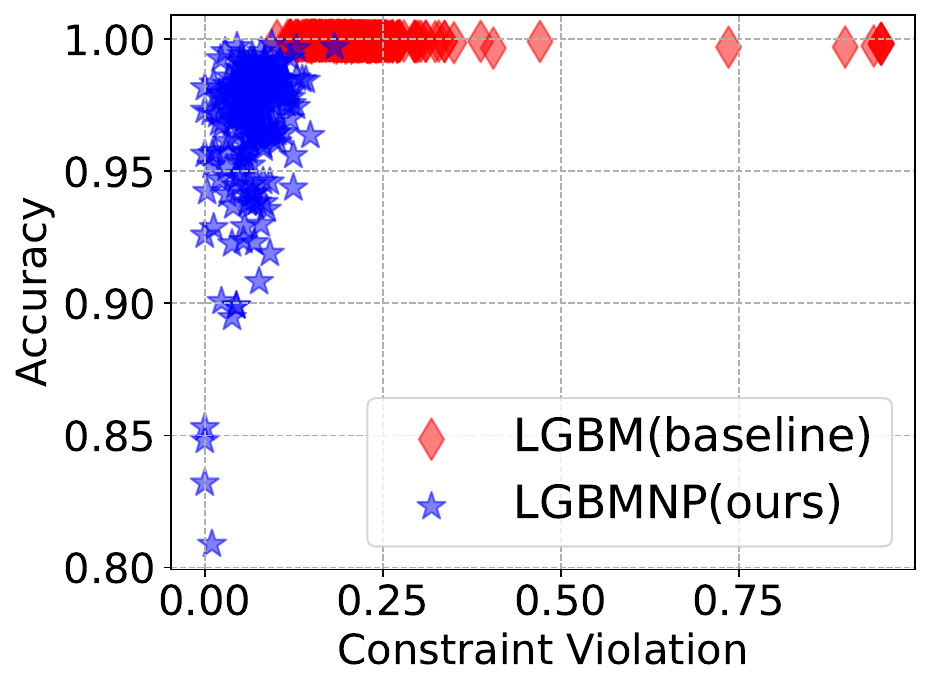}%
\end{minipage}%
\begin{minipage}[t]{0.25\columnwidth}%
\includegraphics[width=4.25cm]{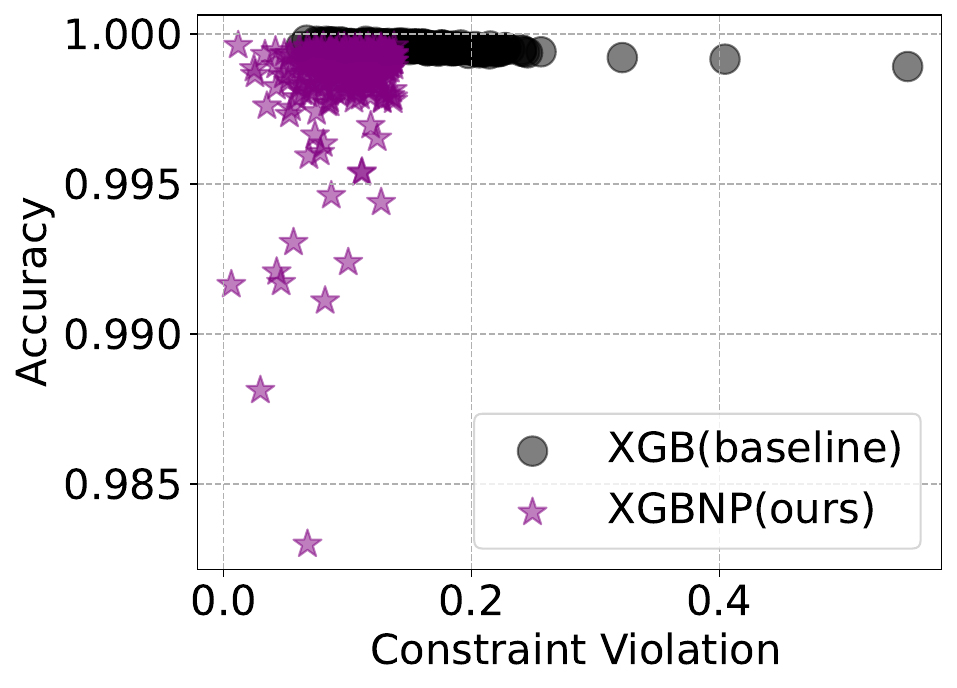}%
\end{minipage}%
\begin{minipage}[t]{0.25\columnwidth}%
\includegraphics[width=4.25cm]{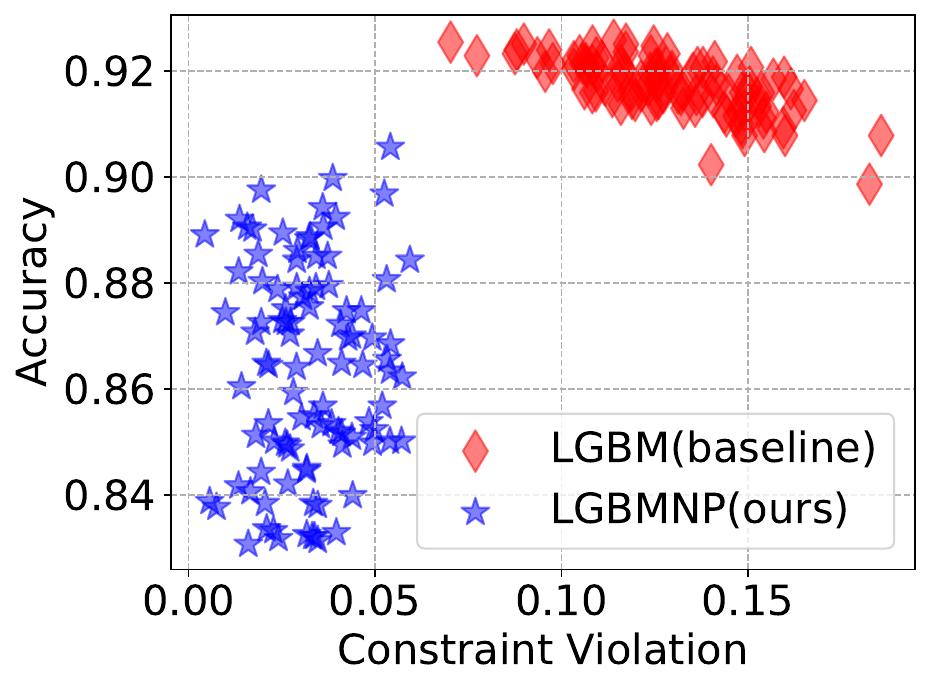}%
\end{minipage}\begin{minipage}[t]{0.25\columnwidth}%
\includegraphics[width=4.25cm]{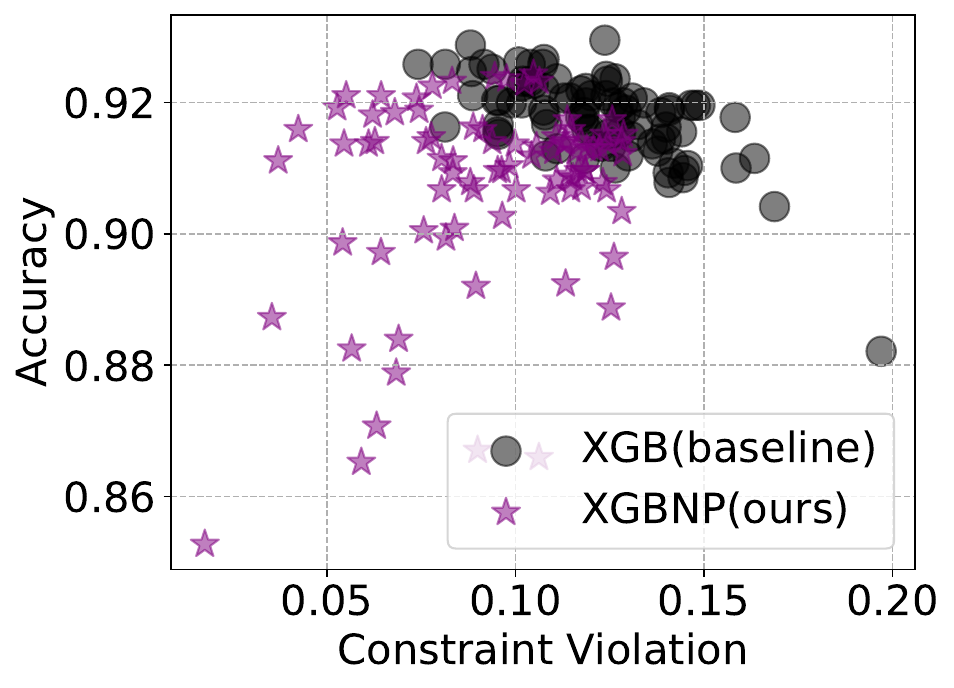}%
\end{minipage}
\end{centering}
\caption{\label{fig:NPC}Result in test set. From left to right: first and second figures - Credit, third and fourth figures - Drybean. The first and third figures represent LightGBM-based algorithms, while the second and fourth figures represent XGBoost-based algorithms.}
\end{figure}

\begin{figure}[]
\begin{centering}
\begin{minipage}[t]{0.25\columnwidth}%
\includegraphics[width=4.25cm]{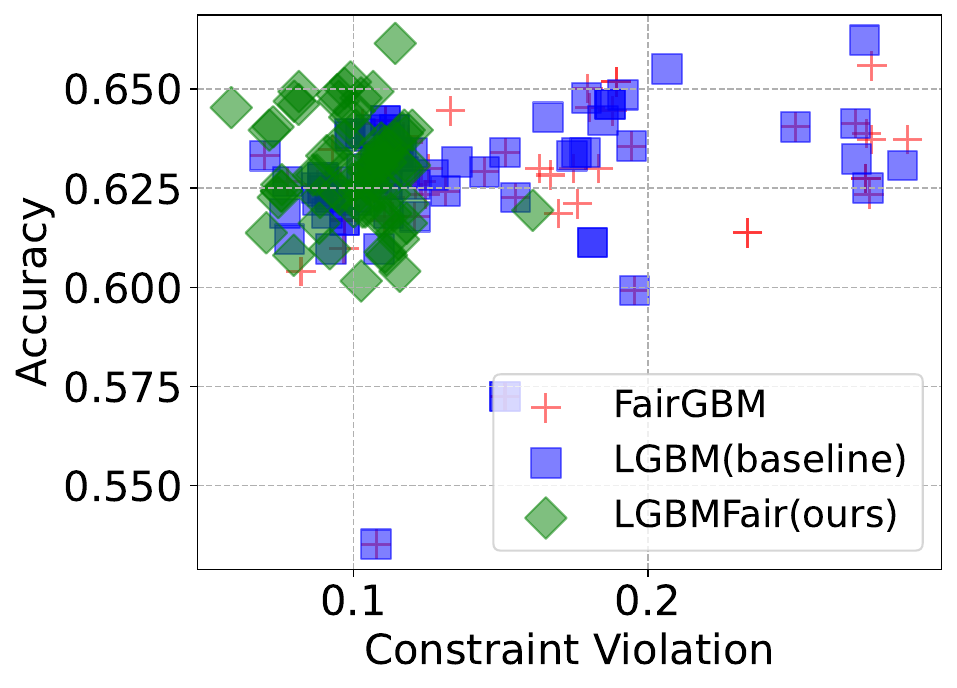}%
\end{minipage}%
\begin{minipage}[t]{0.25\columnwidth}%
\includegraphics[width=4.25cm]{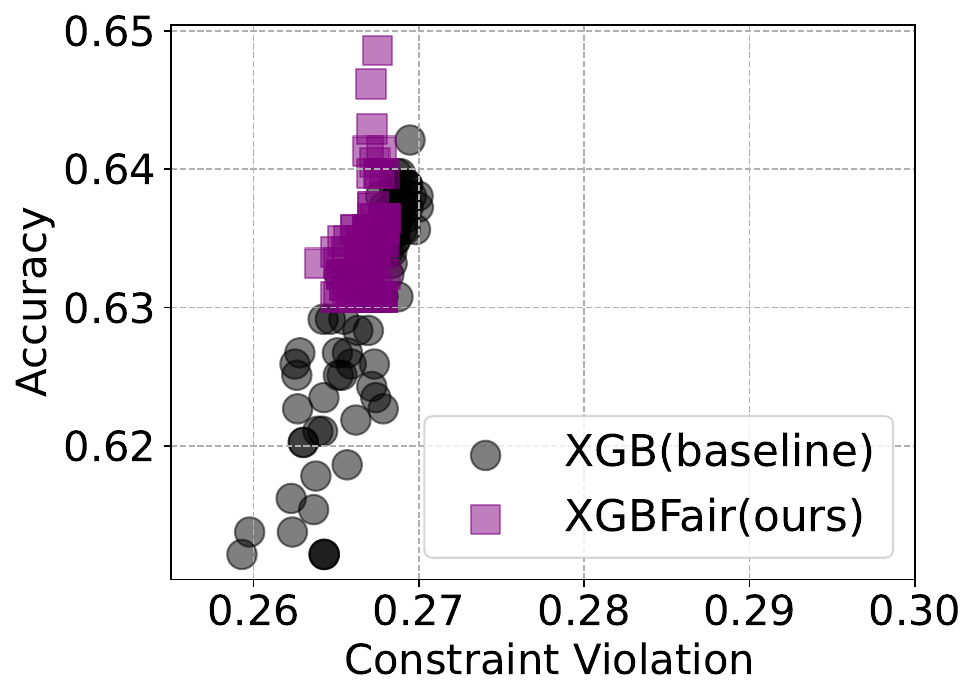}%
\end{minipage}%
\begin{minipage}[t]{0.25\columnwidth}%
\includegraphics[width=4.25cm]{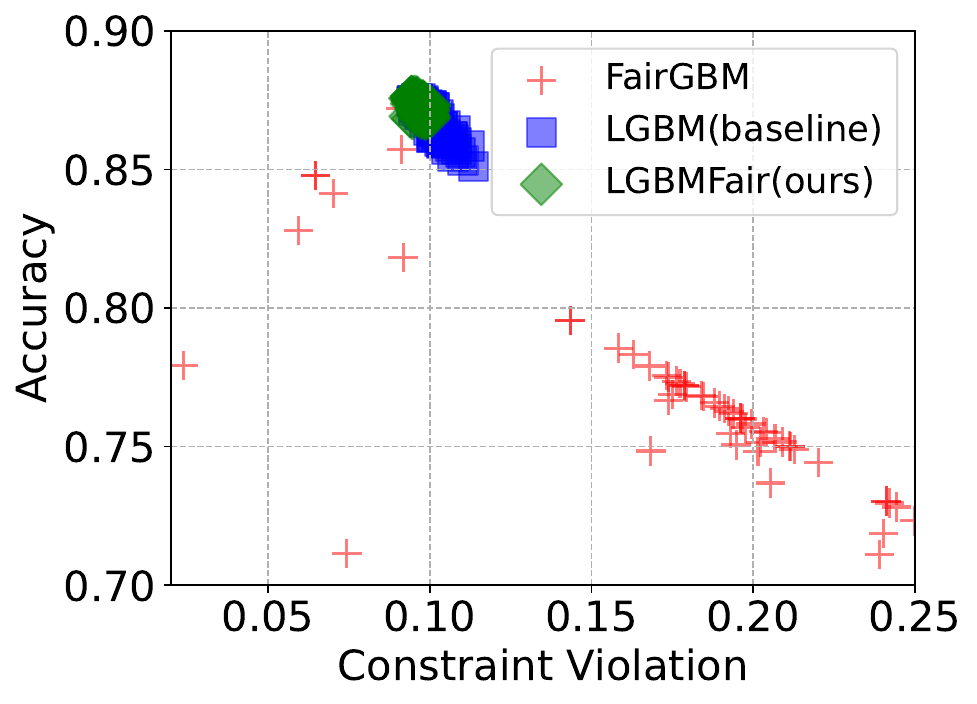}%
\end{minipage}\begin{minipage}[t]{0.25\columnwidth}%
\includegraphics[width=4.25cm]{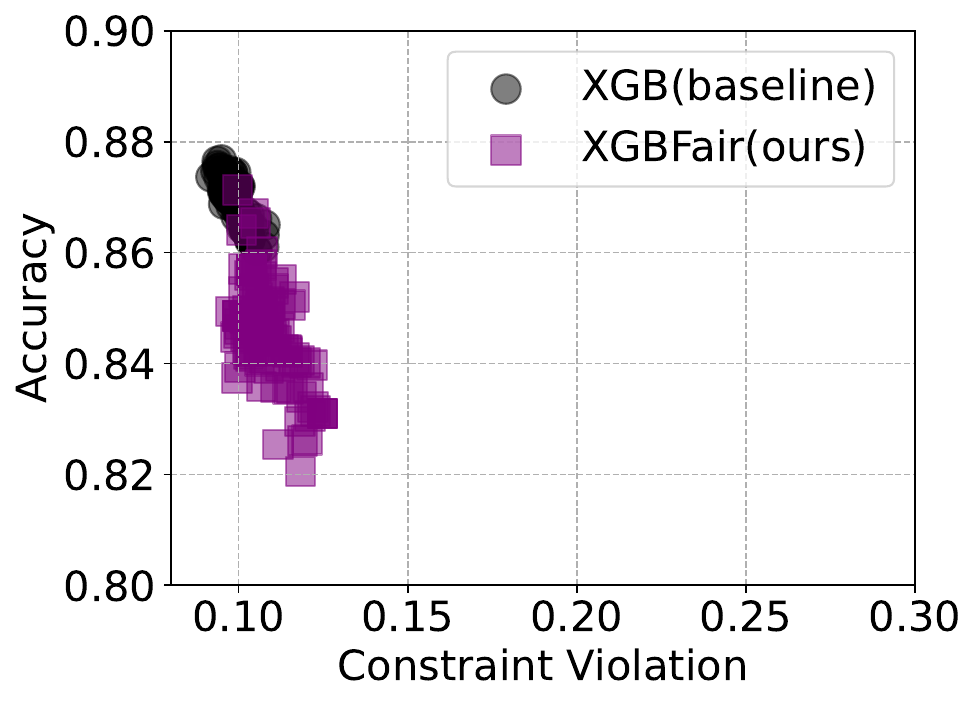}%
\end{minipage}\end{centering}
\caption{\label{fig:fair}Result in test set. From left to right: first and second figures - Compas, third and fourth figures - Adult. The first and third figures represent LightGBM-based algorithms, while the second and fourth figures represent XGBoost-based algorithms.}
\vspace{-1.0em}
\end{figure}

\begin{figure}[]
\begin{center}
\begin{minipage}[t]{0.5\columnwidth}
\centering
\includegraphics[width=6.0cm]{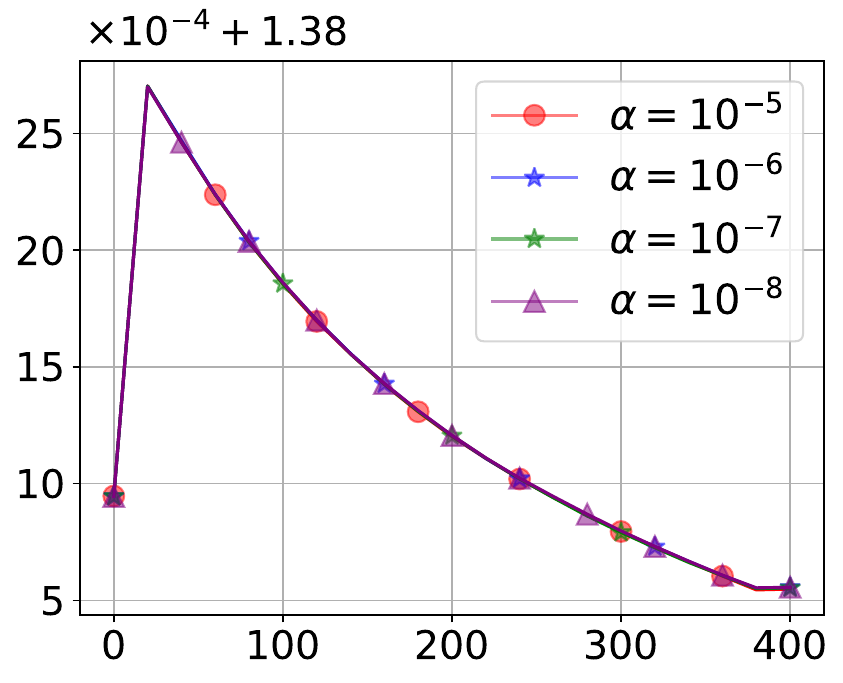}%
\end{minipage}\begin{minipage}[t]{0.5\columnwidth}
\centering
\includegraphics[width=6.0cm]{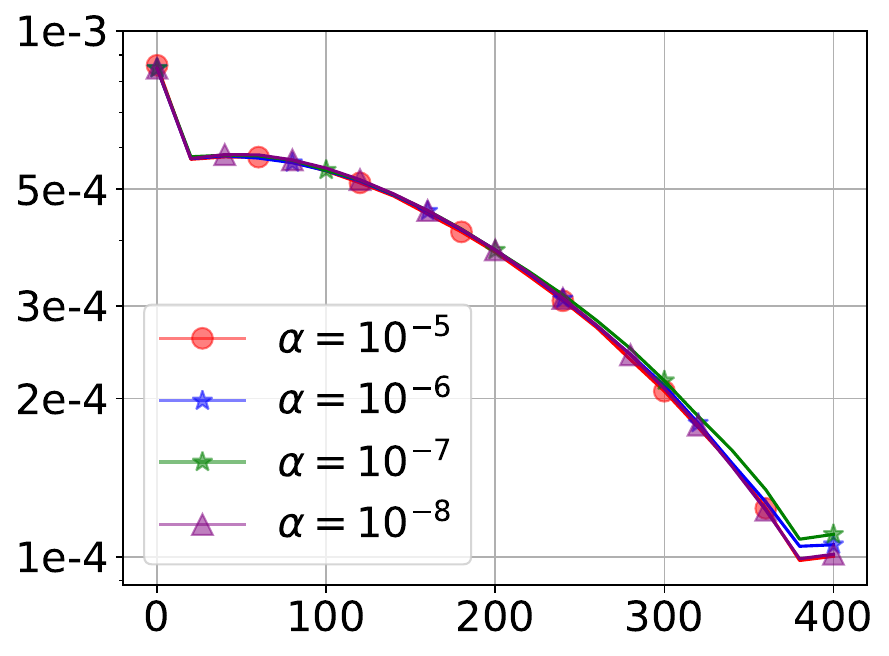}%
\end{minipage}\end{center}
\caption{\label{fig:illustrate} Left (Right): Training loss (loss difference of two groups) vs. Iteration number}
\vspace{-1.5em}
\end{figure}

\paragraph{Credit.} 
The dataset comprises credit card transactions, primarily aimed at detecting fraudulent activities. It's important to highlight that the dataset is significantly imbalanced; out of a total of 284,807 transactions, a mere 492 were found to be fraudulent.
In practice, the costs of misclassifying different classes are highly disparate. In particular, the penalty for incorrectly identifying fraudulent transactions (0) as non-fraudulent (1) is substantially greater. This consideration leads us to the NPC:
\begin{equation*}
    \min_{\phi\in \Phi}\,\, \xi(1)\ \ \st \xi(0)\leq 0.05,
\end{equation*}
where $\Phi$ is the space of classifier, $\xi(i) = \mathbb{\text{\ensuremath{\Pr}}}_{\abf\mid b=i}(\phi(\abf)\neq i)$. It's worth noting that the results from RandF+CX and RandF+ER are excluded from our discussion as these two algorithms yield infeasible outcomes for this particular dataset.

\paragraph{Drybean.} This dataset focuses on a classification problem involving seven bean categories, represented by classes 0 through 6. In each run, the data were randomly divided into 80\% training and 20\% testing for each class. In the task, we consider the following NPC:
\begin{equation*}
\min_{\phi\in \Phi}\sum_{i=0}^{6} \xi(i),\ \ \st [\xi(i)]_{i=1}^4\preceq[1,3,2,2]\cdot 10^{-2}.
\end{equation*} 

Due to the large  data volume in the Credit problem, the algorithms RandF+CX and RandF+ER encounter memory overflow issues, preventing their successful execution. Consequently, their performance on the Credit dataset is not presented in this study.
Figure~\ref{fig:NPC} and Table~\ref{tab:NPC} demonstrate the efficacy of two algorithms, LGBMNP ({\apd}+LightGBM) and XGBNP ({\apd}+\allowbreak XGBoost), in solving the NPC problem. In comparison to the baseline algorithm, both LGBMNP and XGBNP algorithms exhibit smaller constraint violations at the expense of a marginal decline in classification accuracy. It is worth mentioning that our algorithm performs significantly better in terms of classification accuracy compared to the RandF+CX and RandF+ER algorithms. Although this phenomenon may be attributed to the superior nature of the XGBoost and LightGBM classifiers over the Random Forest itself, such a comparison indirectly reflects the superiority of our methods.

\paragraph{Adult.} The Adult dataset is used for binary classification tasks to predict income levels. Our focus is on optimizing global accuracy while ensuring fairness across gender groups. We encode salary and gender into two categories (0-1) to simplify the presentation. Subsequently,
then we address the following fairness problem:
\begin{equation*}
    \min_{\phi\in \Phi} \sum_{i=0}^{1} \xi(i)\ \  \st \sum_{i=0}^{1}\xi_{0}(i) -  
 \xi_{1}(i)\leq \alpha,
\end{equation*}
where $\xi_j(i) = \text{Pr}_{\abf\mid b=i,s=j}(\phi(\abf)\neq i)$.

\paragraph{Compas.} The dataset includes defendant information from Broward County, Florida, for predicting reoffending risk. Our goal is to balance outcomes across races when optimizing. The task involves addressing a fairness problem:
\begin{equation*}
\min_{\phi\in \Phi}\sum_{i=0}^{2}\xi_j(i) \,
\textnormal{s.t.}~ \sum_{i=0}^{2}\xi_{j_{1}}(i)-\xi_{j_{2}}(i)\leq\alpha,\forall j_{1},j_{2}\in[5].
\end{equation*}

Figure~\ref{fig:fair} and Table~\ref{tab:fair} illustrate the effectiveness of both LGBMFair ({\ilcp}+LightGBM) and XGBFair ({\ilcp}+XGBoost) algorithms for addressing the fairness problem. 
We summarize the experimental results in terms of LightGBM-based and XGBoost-based models.
Among the three models based on LightGBM, LGBMFair consistently outperforms other models in accuracy and fairness. Interestingly, the baseline model exhibits better results than the FairGBM. On the other hand, XGBFair has better accuracy than the XGBoost baseline model in the Compas dataset, and there is no significant loss in the fairness measure from Figure~\ref{fig:fair} and Table~\ref{tab:fair}. However, it appears that there is still room to improve the stability of XGBFair based on its performance on the Adult dataset.
On the Adult dataset, its algorithm performs worse than the baseline in both measures. To investigate this, we explored the changes in the two groups' objective loss and loss differences during the optimization process. By setting different $\alpha$, we obtained Figure~\ref{fig:illustrate}. The experimental results show a significant decrease in the objective function and the loss difference, both of which end up lower than the baseline (i.e., the initial iteration point). However, the classification accuracy and fairness results, shown in Table~\ref{tab:fair}, do not surpass the baseline. We speculate that this phenomenon may be due to the XGBoost model's parameter settings and the problem's non-convex nature.

\section{\label{sec:conclusion}Conclusion}
The main contribution of this paper is the development of several Bregman proximal algorithms for addressing constrained learning, specifically in the context of NPC and fairness in ML. We have demonstrated the compatibility of our algorithm with several mainstream classifiers and have made our LGBMNP and XGBNP for solving the NPC problem, as well as LGBMFair, and XGBFair for the fairness problem, openly available.
In future research, we aim to extend the application range of the Bregman proximal algorithm to encompass a wider array of constrained learning problems, such as shape-restricted regression~\citep{sen2017testing} and constrained semi-supervised learning~\citep{ouali2020overview}.
Additionally, exploring non-convex loss functions, such as the sum of ranked range~\citep{hu2020learning}, would be an interesting avenue for future investigation.

%% file: supp_main.tex
\section{Additional Literature Review}
\paragraph{Neyman Pearson Classification} NPC is beyond the scope of standard machine learning solvers such as scikit-learn~\cite{scikit-learn},
GBM~\cite{chen2016xgboost,ke2017LightGBM},
which are primarily designed for unconstrained or projection-friendly learning models. 
NPC has attracted growing research interest~\cite{cannon2002learning,scott2005neyman,tong2020neyman,mossman1999three,dreiseitl2000comparing,landgrebe2005neyman}.
Refer to~\cite{tong2016survey} for a more extensive survey. 
Due to the difficulty in directly minimizing classification
error rates, \citet{rigollet2011neyman,han2008analysis} proposed
replacing the 0-1 loss function with convex surrogate functions. \citet{tong2018neyman}
presented an umbrella algorithm adaptable to popular methods
such as logistic regression, SVM, and random forests. 
\citet{tong2020neyman}
proposed a parametric linear discriminant analysis for NPC and developed
a novel approach to choose the classification threshold value. 
However, although much progress has been made in NP binary classification,there has been far less work on the more challenging Neyman-Pearson multiclass classification (NPMC) problem. 
Recently, \citet{tian2021neyman}
considered a more general multiclass classification (NPMC) problem,
demonstrating that a desired strong duality holds for NPMC under
mild assumptions. 
Consequently, they proposed to perform optimization in the
dual space.
However, since the dual objective implicitly involves a
nontrivial error minimization problem, it remains unclear how efficient this algorithm is. 
It should be noted that, aside from NPC, cost-sensitive learning~\cite{elkan2001foundations,zadrozny2003cost,li2020bridging} is another important paradigm to address the asymmetric error control, wherein practitioners manually assign specific weights to misclassification errors for each category. 
However, the drawback of this method lies in the challenge of determining the right weight for the target cost.
Moreover, the weight assigned to one class may influence other classes through the regularized formulation, further complicating the optimization process.

\paragraph{Function constrained optimization}There is a vast literature on convex function constrained optimization~\cite{yang2017richer,xu2021firstordera,bayandina2018mirror}. Popular methods include  Augmented Lagrangian methods~\cite{lan2016iterationcomplexity}, level-set methods~\cite{lin2018levelset}, penalty  methods~\cite{lan2013iterationcomplexitya} and primal dual method~\cite{hamedani2021primal}.
The iterative format of the Accelerated Primal-dual algorithm (APD)~\cite{hamedani2021primal}  is relatively simple, comprising just three steps: gradient extrapolation, dual ascent using an explicit expression, and proximal gradient descent for the primal variable. 
For nonconvex constrained optimization, penalty methods~\cite{cartis2014complexity,wang2017penalty,cartis2011evaluation} is one popular approach for constrained optimization. Furthermore, \cite{boob2023stochastic,ma2020quadratically,lin2019inexact,boob2022level} presented some new proximal point algorithms that iteratively solve the strongly convex proximal sub-problem inexactly by first-order methods. \citet{jia2022first} employs the technique of directly inserting quadratic terms into the objective function and constraints of the original problem to convert it into a convex problem without requiring linearization. Theoretically, \citet{jia2022first} guarantees convergence to the Fritz John (FJ) point without any constraint qualification.

\section{Preliminaries}
List all corresponding materials here for completeness.
\begin{equation}
\begin{split}\min_{\xbf\in\mcal X} & \ \ \phi_{0}(\xbf):=f(\xbf)\\
\st & \ \ \bar{\phi}(\xbf):=\max_{i\in[m]}\{g_{i}(\xbf)-h_{i}(\xbf)-\eta_{i}\}\leq0,
\end{split}
\label{eq:fair_prob_appendix}
\end{equation}
\begin{algorithm}[htp]
\caption{\label{alg:APDinexact_appendix}Accelerated Bregman primal-dual Proximal Point algorithm ({\apd})}
    \begin{algorithmic}[1]
        \REQUIRE{$\tau_0>0, \sigma_0>0$, $\mathbf{x}_0,\mathbf{y}_0,\mu$}
        \STATE{\textbf{Initialize:} $\mathbf{x}_{-1}\leftarrow \mathbf{x}_0, \mathbf{y}_{-1}\leftarrow \mathbf{y}_{0}$,$\sigma_{-1}\leftarrow\sigma_0$,$\gamma_0=\sigma_0/\tau_0$,$k=0$}
        \WHILE{$k<K$}
		\STATE{$\sigma_k \leftarrow \gamma_k \tau_k$,$\theta_k \leftarrow \frac{\sigma_{k-1}}{\sigma_k}$}

        \STATE{$\mathbf{z}_{k}\leftarrow (1+\theta_k)\mathbf{g}(\mathbf{x}_k) -  \theta_k \mathbf{g}(\mathbf{x}_{k-1})$}
\STATE{
$\mathbf{y}_{k+1}\leftarrow [\ybf_k + \sigma_k \zbf_k]_+$
}

\STATE{Find a $(\delta_{k+1},\nu_{k+1})$-approximate solution $\mathbf{x}_{k+1}$   for sub-problem below $$\mathbf{x}_{k+1}\approx \operatornamewithlimits{argmin}_{\mathbf{x}\in\mathcal{X}}f(\mathbf{x})+\mathbf{y}_{k+1}^T\mathbf{g}(\mathbf{x}) + \frac{1}{\tau_k}\mathbf{D}_{\mathbf{W}}(\mathbf{x},\mathbf{x}_k)$$\label{lst:line:minphi_inexact_appendix}}
\STATE{$\gamma_{k+1}\leftarrow \gamma_k(1+\mu \tau_k)$,$\tau_{k+1}\leftarrow \tau_k \sqrt{\frac{\gamma_k}{\gamma_{k+1}}}$,$k\leftarrow k + 1$}
        \ENDWHILE
    \STATE{\textbf{Output: $\mathbf{x}_{K}$}}
        \end{algorithmic}
\end{algorithm}
\begin{algorithm}[htp]
\caption{\label{alg:LCP_appendix} Constrained Bregman Proximal Regularized method ({\ilcp})}
\begin{algorithmic}
\REQUIRE{A feasible point $\xbf^0$ for problem~\eqref{eq:fair_prob_appendix}, $L>\max\bcbra{\max_{i}\bcbra{L_{h_{i}}},\sigma_{\max}(\mathbf{W}^{\top} \mathbf{W})^{-1}}$}
        \FOR{$t=0,1,\ldots, T-1$ } 
        \STATE{Call {\apd} to find an $(\vep_3,\vep_4)$ optimal solution $\xbf^{t+1}$ (Definition~\ref{defn:optimal_appendix}) for the following problem
\begin{equation}\label{eq:lcp_subprob_appendix} \begin{split}\min_{\xbf\in\mcal X} & \ \ \phi_t^0(\xbf):=f(\xbf) + L\mathbf{D}_{\mathbf{W}}(\xbf,\xbf^t)\\ \st & \ \ \bar{\phi}_t(\xbf):=\bar{\phi}(\xbf)+L\mathbf{D}_{\mathbf{W}}(\xbf,\xbf^{t})\leq 0\end{split} \end{equation}
}
\ENDFOR
\STATE{\textbf{Output:} $\xbf^{T}$}
\end{algorithmic}
\end{algorithm}
\begin{defn}
\label{def:approx_stochastic_appendix}Define a 
$(\delta_{k+1},\nu_{k+1})-$approximate
solution $\xbf_{k+1}$ of line~\ref{lst:line:minphi_inexact_appendix} in {\apd}
as follows:
\begin{equation}
\psi_{k}(\xbf_{k+1})-\psi_{k}(\hat{\xbf}_{k+1})\leq\delta_{k+1}\mathbf{D}_{\mathbf{W}}(\hat{\xbf}_{k+1},\xbf_{k})+\nu_{k+1},\ \ \xbf_{k+1}\in \mcal X,\label{eq:sub-acc_appendix}
\end{equation}
where $\psi_{k}(\xbf)=f(\xbf)+\left\langle \ybf_{k+1},\mathbf{g}(\xbf)\right\rangle +\frac{1}{\tau_{k}}\mathbf{D}_{\mathbf{W}}(\xbf,\text{\ensuremath{\xbf}}_{k})$,
$\hat{\xbf}_{k+1}$ is the exact solution of line~\ref{lst:line:minphi_inexact_appendix}
in {\apd}.
\end{defn}

\begin{defn}[$(\vep_3,\vep_4)$-optimal]\label{defn:optimal_appendix}
We say $\xbf$ is an $(\vep_{3},\vep_{4})$
optimal solution for problem~\eqref{eq:lcp_subprob_appendix} if $\phi_{t}^{0}(\xbf)-\phi_{t}^{0}(\xbf^{*})\leq\vep_{3}$
and $\bar{\phi}_{t}(\xbf)\leq\vep_{4}$, where $\xbf^{*}$ is an
optimal solution.
\end{defn}
\begin{defn}[FJ point]\label{def:optimal_fj_appendix}
 We say a feasible $\xbf^{*}$ is a FJ point of~\eqref{eq:fair_prob_appendix}
if there exists subgradient $d_{\phi0}\in\partial\phi_{0}(\xbf^*)$, $d_{\phi_i}\in\partial (g_i(\xbf^*) - h_i(\xbf^*))$ and nonnegative multipliers $y_{0}^{*}\in\mbb R_{+}$
and $\ybf^{*} = [y_1^*, y_2^*,\cdots, y_m^*]^T\in\mbb R_{+}^m$ such that
\begin{align}
y_i^{*} (g_i(\xbf^*)-h_i(\xbf^*)-\eta_i) & =0,\ \ \forall i \in [m]\label{eq:fj_point_1_appendix}\\
y_{0}^{*}d_{\phi_{0}}+\sum_{i=1}^{m} y_i^{*}d_{\phi_i} & \in-\mcal N_{\mcal X}(\xbf^{*}).\label{eq:fj_point_2_appendix}
\end{align}
\end{defn}
\begin{defn}[KKT point] We say a feasible $\xbf^*$ is a KKT point of~\eqref{eq:fair_prob_appendix} if there exists subgradient $d_{\phi 0}\in \partial \phi_0(\xbf^*)$, $d_{\phi_i}\in \partial (g_i(\xbf^*)-h_i(\xbf^*))$ and $\ybf^*=[y_1^*, y_2^*,\cdots, y_m^*]^T\in\mbb R_{+}^m$ such that
\begin{align}
y_i^{*} (g_i(\xbf^*)-h_i(\xbf^*)-\eta_i) & =0,\ \ \forall i \in [m]\label{eq:fj_point_1_appendix}\\
d_{\phi_{0}}+\sum_{i=1}^{m} y_i^{*}d_{\phi_i} & \in-\mcal N_{\mcal X}(\xbf^{*}).
\end{align}
\end{defn}
\begin{thm}
\label{thm:inexact_thm_expected_appendix}
Suppose $f$ is $\mu$-$\mathbf{D}_{\mathbf{W}}$ relatively strongly convex, $\gbf$ is $L_g$-$\mathbf{D}_{\mathbf{W}}$ relatively Lipschitz continuous, 
and there exists a constant $\delta>0$ such that the sequence $\{\tau_{k},\sigma_{k},t_{k}\}_{k\geq0}$
satisfies
\begin{equation}
\begin{split}t_{k+1}\tau_{k+1}^{-1}\leq t_{k}(\tau_{k}^{-1}+\mu),\ \  & t_{k+1}\sigma_{k+1}^{-1}\leq t_{k}\sigma_{k}^{-1},\ \ t_{k+1}\theta_{k+1}=t_{k},\\
L_{g}\sigma_{k}/\delta\leq(\tau_{k})^{-1} & ,\ \ \theta_{k}\delta/\sigma_{k-1}\leq\sigma_{k}^{-1},\ \ \theta_{k}\sigma_{k-1}=\sigma_{k}.
\end{split}
\label{eq:parameter_setting}
\end{equation}
Then,  for the sequence $\{\xbf_{k},\ybf_{k},\bar{\xbf}_{k},\bar{\ybf}_{k}\}$
generated by {\apd}, we have
\begin{equation*}
T_{K}[\mcal L(\bar{\xbf}_{K,}\ybf)-\mcal L(\xbf,\bar{\ybf}_{K})+t_{K-1}\tau_{K-1}^{-1}\mathbf{D}_{\mathbf{W}}(\xbf,\xbf_{K})]\leq \tfrac{t_{0}\mathbf{D}_{\mathbf{W}}}{\tau_{0}}(\xbf,\xbf_{0})+\tfrac{t_{0}\sigma_{0}^{-1}}{2}\norm{\ybf-\ybf_{0}}^{2}+\mbb \sum_{k=0}^{K-1}t_{k}\eta_{k},\label{eq:thm1_converge}
\end{equation*}
where $\eta_{k}:=2\sqrt{\tfrac{1}{\tau_{k}}+\mu}\sqrt{\mathbf{D}_{\mathbf{W}}(\xbf,\xbf_{k+1})}\sqrt{\delta_{k+1}\mathbf{D}_{\mathbf{W}}(\hat{\xbf}_{k+1},\xbf_{k})+\nu_{k+1}}$.
\end{thm}

\begin{cor}\label{cor:inexact_cor_expected_appendix} Denote $\ybf^{+}:=\frac{1}{\sqrt{m}}(\|\ybf^{*}\|_{1}+1)\frac{[\mathbf{g}(\bar{\xbf}_{K})]_{+}}{\|[\mathbf{g}(\bar{\xbf}_{K})]_{+}\|}$,
$\Delta_{*+}:=\frac{1}{\tau_{0}}\mathbf{D}_{\mathbf{W}}(\mathbf{x}^{*},\mathbf{x}_{0})+\frac{1}{2\sigma_{0}}\|\mathbf{y}^{+}-\mathbf{y}{}_{0}\|^{2}$. For a compact set $\Xcal$, we denote $D_{X}:=\max_{\xbf_{1},\xbf_{2}\in\mathcal{X}}\sqrt{2\mathbf{D}_{\mathbf{W}}(\xbf_{1},\xbf_{2})}<\infty$.
Suppose that $f$ is $\mu$-$\mathbf{D}_{\mathbf{W}}$ relatively strongly convex in $\mcal X$, $\gbf$ is $L_g$-$\mathbf{D}_{\mathbf{W}}$ relatively Lipschitz continuous in $\mcal X$ and sequence $t_{k}=\sigma_{k}/\sigma_{0}$ and $\{\tau_{0},\sigma_{0}\}$
satisfies $\tau_{0}\sigma_{0}\leq\delta/L_{g}.$ Then for  the sequence $\{\xbf_{k},\ybf_{k},\bar{\xbf}_{k},\bar{\ybf}_{k}\}$
generated by {\apd}, we have 
\begin{equation*}
\begin{split} & \max\{f(\bar{\xbf}_{K})-f(\xbf^{*}) ,\|[\mathbf{g}(\bar{\xbf}_{K})]_{+}\|\}\leq\tfrac{\Delta_{*+}}{T_{K}}+\tfrac{\sum_{k=0}^{K-1}t_{k}\eta_{k}}{T_{K}},\end{split}
\label{eq:cor_01_appendix}
\end{equation*}
where $\Delta_{*+}:=\tfrac{1}{\tau_{0}}\mathbf{D}_{\mathbf{W}}(\xbf_{0},\xbf^{*})+\tfrac{1}{2\sigma_{0}}\norm{\ybf^{+}-\ybf_{0}}^{2}$.
Furthermore, take $\delta_{k+1}=\nu_{k+1}=\tfrac{\tau_{0}}{(k+2)^{7}}$
when $\mu>0$, then we can get an $\vep$-optimal solution in at most $K\geq\sqrt{\tfrac{6(\mu\tau_{0})^{-1}(\Delta_{*+}+D_{X}^{2}+\sqrt{2}D_{X})}{\vep}}+1$
iterations. Take $\delta_{k+1}=\nu_{k+1}=\tfrac{\tau_{0}}{(k+2)^{4}}$
when $\mu=0$, then we can get an $\vep-$optimal solution in at most $\tfrac{\Delta_{*+}+D_{X}^{2}+\sqrt{2}D_{X}}{\vep}$
iterations.
\end{cor}

\begin{thm}
\label{thm:3.2_appendix}
Suppose $\inf_{\xbf\in\mcal X}f(\xbf)>-\infty$ and taking $\vep_{3}=\vep_{4}=\tfrac{(L-\rho)\vep^{2}\sigma_{\min}(\mathbf{W}^{\top} \mathbf{W})}{4L^{2}\sigma_{\max}(\mathbf{W}^{\top} \mathbf{W})^{2}}$
in {\ilcp}, where $\rho = \max_{i}\bcbra{L_{h_i}}$. Then, {\ilcp} return $(\vep, \vep)$-FJ point at most 
\begin{equation*}
T:=\tfrac{4L^{2}\sigma_{\max}(\mathbf{W}^{\top} \mathbf{W})^{2}(f(\xbf^{0})-\inf_{\xbf\in\mcal X}f(\xbf))}{3(L-\rho)\sigma_{\min}(\mathbf{W}^{\top} \mathbf{W})\vep^{2}}-1
\end{equation*}
outer iterations, and the overall iteration is $\mcal{O}(1/\vep^3)$ when using {\apd} to solve~\eqref{eq:lcp_subprob_appendix}. 
\end{thm}

\section{Auxiliary Lemmas}
\begin{lem}[Property 1 in~\cite{tseng2008accelerated}]
\label{lem:three_point}Let $f:\mathbb{R}^{n}\to\mathbb{R\cup}\{+\infty\}$
be a closed strongly convex function with modulus $\mu\geq0$. Give
$\bar{\xbf}\in\mathcal{X}$ where $\mathcal{X}$ is a compact convex
set and $t\geq0$, let $\xbf^{+}=\argmin_{x\in\mathcal{X}}f(\xbf)+t\mathbf{D}_{\Wbf}(\xbf,\bar{\xbf})$.
Then, for all $\xbf$, the following inequality holds $f(\xbf^{+})+t\mathbf{D}_{\Wbf}(\xbf^{+},\xbf)+t\mathbf{D}_{\Wbf}(\xbf^{+},\bar{\xbf})+\mu\mathbf{D}_{\Wbf}(\xbf,\xbf^{+})\le f(\xbf)+t\mathbf{D}_{\Wbf}(\xbf,\bar{\xbf})$.
\end{lem}

\begin{lem}
\label{lem:Approximate_sol}Suppose that $\xbf_{k+1}$ is a stochastic
$(\delta_{k+1},\nu_{k+1})-$approximate solution defined in Definition~\ref{def:approx_stochastic_appendix},  $f$ is $\mu-\Dbf_W$ Relative Strong Convexity and $\gbf$ is $L_g$ Relative Lipschitz Continuity, then for $\forall \xbf \in \mathcal{X}$, we have
\begin{equation}\label{eq:Approximate_sol}
\begin{split} & f(\xbf_{k+1})-f(\xbf)+\inprod{\ybf_{k+1}}{\gbf(\xbf_{k+1})-\gbf(\xbf)}\\
\leq & \tfrac{1}{\tau_{k}}\brbra{\mathbf{D}_{\Wbf}(\xbf,\xbf_{k})-\mathbf{D}_{\Wbf}(\xbf_{k+1},\xbf_{k})}\\
 & -(\tfrac{1}{\tau_{k}}+\mu)\mathbf{D}_{\Wbf}(\xbf,\xbf_{k+1})+2\sqrt{\tfrac{1}{\tau_{k}}+\mu}\sqrt{\mathbf{D}_{\Wbf}(\xbf,\xbf_{k+1})}\sqrt{\delta_{k+1}\mathbf{D}_{\Wbf}(\hat{\xbf}_{k+1},\xbf_{k})+\nu_{k+1}}\\
 & +\delta_{k+1}\mathbf{D}_{\Wbf}(\hat{\xbf}_{k+1},\xbf_{k})+\nu_{k+1},
\end{split}
\end{equation}where $\hat{\xbf}_{k+1}$ is the exact optimal solution of $\min_{\xbf\in\mcal X}f(\xbf)+\inprod{\ybf_{k+1}}{\gbf(\xbf)}+\tfrac{1}{\tau_{k}}\mathbf{D}_{\Wbf}(\xbf,\xbf_{k})$.
\end{lem}
\begin{proof}
    It follows from the definition of $\hat{\xbf}_{k+1}$ and Lemma~\ref{lem:three_point} that
\begin{align}
 & f(\hat{\xbf}_{k+1})-f(\xbf)+\inprod{\ybf_{k+1}}{\gbf(\hat{\xbf}_{k+1})-\gbf(\xbf)}\nonumber \\
\leq & \tfrac{1}{\tau_{k}}\brbra{\mathbf{D}_{\Wbf}(\xbf,\xbf_{k})-\mathbf{D}_{\Wbf}(\xbf,\hat{\xbf}_{k+1})-\mathbf{D}_{\Wbf}(\hat{\xbf}_{k+1},\xbf_{k})}-\mu\mathbf{D}_{\Wbf}(\xbf,\hat{\xbf}_{k+1}),\forall\xbf\in\mcal X.\label{eq:three_point_x-2}
\end{align}

Placing $\xbf=\xbf_{k+1}$ in~\eqref{eq:three_point_x-2}, then we have

\begin{align}
 & f(\hat{\xbf}_{k+1})-f(\xbf_{k+1})+\inprod{\ybf_{k+1}}{\gbf(\hat{\xbf}_{k+1})-\gbf(\xbf_{k+1})}\nonumber \\
\leq & \tfrac{1}{\tau_{k}}\brbra{\mathbf{D}_{\Wbf}(\xbf_{k+1},\xbf_{k})-\mathbf{D}_{\Wbf}(\xbf_{k+1},\hat{\xbf}_{k+1})-\mathbf{D}_{\Wbf}(\hat{\xbf}_{k+1},\xbf_{k})}-\mu\mathbf{D}_{\Wbf}(\xbf_{k+1},\hat{\xbf}_{k+1}),\forall\xbf\in\mcal X.\label{eq:three_point_x-3}
\end{align}

By the definition of $(\delta_{k+1},\nu_{k+1})-$approximate solution,
we have
\begin{equation}
\begin{split} & f(\xbf_{k+1})-f(\hat{\xbf}_{k+1})+\inprod{\ybf_{k+1}}{\gbf(\xbf_{k+1})-\gbf(\hat{\xbf}_{k+1})}\\
\leq & \tfrac{1}{\tau_{k}}\mathbf{D}_{\Wbf}(\hat{\xbf}_{k+1},\xbf_{k})-\tfrac{1}{\tau_{k}}\mathbf{D}_{\Wbf}(\xbf_{k+1},\xbf_{k})\\
 & +\delta_{k+1}\mathbf{D}_{\Wbf}(\hat{\xbf}_{k+1},\xbf_{k})+\nu_{k+1}
\end{split}
\label{eq:approximate}
\end{equation}

Summing~\eqref{eq:three_point_x-3} and~\eqref{eq:approximate},
we have
\begin{equation}
\begin{split}(\tfrac{1}{\tau_{k}}+\mu)\mathbf{D}_{\Wbf}(\xbf_{k+1},\hat{\xbf}_{k+1}) & \leq\end{split}
\delta_{k+1}\mathbf{D}_{\Wbf}(\hat{\xbf}_{k+1},\xbf_{k})+\nu_{k+1}.\label{eq:appro_result}
\end{equation}
Combining~\eqref{eq:approximate} and~\eqref{eq:three_point_x-2}, we obtain
\begin{align}
 & f(\xbf_{k+1})-f(\xbf)+\inprod{\ybf_{k+1}}{\gbf(\xbf_{k+1})-\gbf(\xbf)}\nonumber \\
\leq & \tfrac{1}{\tau_{k}}\brbra{\mathbf{D}_{\Wbf}(\xbf,\xbf_{k})-\mathbf{D}_{\Wbf}(\xbf,\hat{\xbf}_{k+1})-\mathbf{D}_{\Wbf}(\xbf_{k+1},\xbf_{k})}-\mu\mathbf{D}_{\Wbf}(\xbf,\hat{\xbf}_{k+1})\label{eq:three_point_x-4}\\
 & +\delta_{k+1}\mathbf{D}_{\Wbf}(\hat{\xbf}_{k+1},\xbf_{k})+\nu_{k+1},\forall\xbf\in\mcal X.\nonumber 
\end{align}
Since $-\tfrac{1}{2}\norm{a+b}_{\mathbf{W}^{\top} \mathbf{W}}^{2}\leq-\tfrac{1}{2}\norm a_{\mathbf{W}^{\top} \mathbf{W}}^{2}+\norm a_{\mathbf{W}^{\top} \mathbf{W}}\norm b_{\mathbf{W}^{\top} \mathbf{W}}$
with $a=\xbf-\xbf_{k+1}$ and $b=\xbf_{k+1}-\hat{\xbf}_{k+1}$, then
we have
\begin{align}
 & f(\xbf_{k+1})-f(\xbf)+\inprod{\ybf_{k+1}}{\gbf(\xbf_{k+1})-\gbf(\xbf)}\nonumber \\
\leq & \tfrac{1}{\tau_{k}}\brbra{\mathbf{D}_{\Wbf}(\xbf,\xbf_{k})-\mathbf{D}_{\Wbf}(\xbf_{k+1},\xbf_{k})}\label{eq:three_point_x-5}\\
 & -(\tfrac{1}{\tau_{k}}+\mu)\mathbf{D}_{\Wbf}(\xbf,\xbf_{k+1})+2(\tfrac{1}{\tau_{k}}+\mu)\sqrt{\mathbf{D}_{\Wbf}(\xbf,\xbf_{k+1})\mathbf{D}_{\Wbf}(\xbf_{k+1},\hat{\xbf}_{k+1})}\nonumber \\
 & +\delta_{k+1}\mathbf{D}_{\Wbf}(\hat{\xbf}_{k+1},\xbf_{k})+\nu_{k+1},\forall\xbf\in\mcal X.\nonumber 
\end{align}
Combining~\eqref{eq:three_point_x-5} and~\eqref{eq:appro_result},
we obtain the desired result.
\end{proof}

\begin{lem}
Suppose $\mu>0$, $t_{k}=\tfrac{\sigma_{k}}{\sigma_{0}}$ and $T_{k}=\sum_{k=0}^{K-1}t_{k},$
then the sequence $\{t_{k},\tau_{k},T_{k}\}$ satisfies
\begin{equation}\label{eq:lower_T_k}
 T_{k}\geq1+\tfrac{\mu\tau_{0}(1+k)k}{6}.
\end{equation}
Moreover, if $\mu\tau_{0}\leq2$, we have
\begin{equation}\label{eq:upper_bound_t_k_tau_k}
\tfrac{t_{k+1}}{\sqrt{\tau_{k}}}\leq\tfrac{(k+2)^{3/2}}{\sqrt{\tau_{0}}}.
\end{equation}
\end{lem}
\begin{proof}
Note that $\tau_{k+1}=\tau_{k}\sqrt{\tfrac{\gamma_{k}}{\gamma_{k+1}}}$
implies that $\tau_{k}=\tau_{0}\sqrt{\tfrac{\gamma_{0}}{\gamma_{k}}}$.
Due to $\gamma_{k+1}=\gamma_{k}(1+\mu\tau_{k})$, we conclude that
$\gamma_{k+1}=\gamma_{k}(1+\mu\tau_{k})=\gamma_{k}+\mu\tau_{0}\sqrt{\gamma_{k}\gamma_{0}}$.
Next, we use induction to show $\tfrac{\sigma_{k}}{\tau_{k}}=\gamma_{k}\geq\tfrac{\mu^{2}\tau_{0}^{2}\gamma_{0}}{9}k^{2}+\gamma_{0}$.
It is easy to see that the induction holds for $k=0$ and 1. Suppose
it holds for $k=0,\ldots,K$, then we have 
\begin{equation*}
\begin{split}\gamma_{K+1} & =\gamma_{K}+\mu\tau_{0}\sqrt{\gamma_{k}\gamma_{0}}\\
 & \geq\tfrac{\mu^{2}\tau_{0}^{2}\gamma_{0}}{9}K^{2}+\gamma_{0}+\mu\tau_{0}\gamma_{0}\brbra{\tfrac{\mu\tau_{0}}{3}K}\\
 & \geq\tfrac{\mu^{2}\tau_{0}^{2}\gamma_{0}}{9}(K+1)^{2}+\gamma_{0},
\end{split}
\end{equation*}
which completes our induction. Then combining the induction result
and the relation among $T_{k},t_{k},\sigma_{k},\tau_{k}$ that
\begin{align*}
T_{k} & =\sum_{s=0}^{k-1}t_{s}=\sum_{s=0}^{k-1}\sqrt{\tfrac{\gamma_{s}}{\gamma_{0}}}\geq1+\sum_{s=1}^{k-1}\sqrt{\tfrac{\mu^{2}\tau_{0}^{2}}{9}s^{2}}\\
 & \geq1+\sum_{s=1}^{k-1}\tfrac{\mu\tau_{0}}{3}s=1+\tfrac{\mu\tau_{0}(k-1)k}{6}.\nonumber 
\end{align*}

Next, we use induction to show that $\gamma_{k}\leq\gamma_{0}(k+1)^{2}$
if $\mu\tau_{0}\leq2$. It is obvious that the inequality holds for
$k=0$. Assume the inequality holds for all $k=0,\cdots,K$, then
we have 
\[
\begin{split}\gamma_{K+1} & =\gamma_{K}+\mu\tau_{0}\sqrt{\gamma_{0}\gamma_{K}}\\
 & \leq\gamma_{0}(K+1)^{2}+\mu\tau_{0}\gamma_{0}(K+1)\\
 & \leq\gamma_{0}(K^{2}+(2+\mu\tau_{0})K+\mu\tau_{0}+1)\\
 & \leq\gamma_{0}(K+2)^{2},
\end{split}
\]
where the last inequality is by $\rho\tau_{0}\leq2$. Hence, we have
\begin{equation*}
\begin{split}\tfrac{t_{k+1}}{\sqrt{\tau_{k}}} & =\sqrt{\tfrac{\gamma_{k+1}}{\sigma_{0}/\tau_{0}}}\cdot\tfrac{\gamma_{k}^{1/4}}{\sqrt{\tau_{0}}\gamma_{0}^{1/4}}\leq\tfrac{\gamma_{k+1}^{3/4}}{\sqrt{\sigma_{0}}\gamma_{0}^{1/4}}\leq\tfrac{(k+2)^{3/2}}{\sqrt{\tau_{0}}}.
\end{split}
\end{equation*}
\end{proof}

\begin{lem}
\label{lem:stochastic_error}Suppose $\delta_{k+1}=\nu_{k+1}=\tfrac{\tau_{0}}{(k+2)^{7}}$
for $\mu>0$ and $\delta_{k+1}=\nu_{k+1}=\tfrac{\tau_{0}}{(k+2)^{4}}$
for $\mu=0$, then we have
\begin{equation*}
\begin{split}\sum_{k=0}^{K-1}t_{k}(\delta_{k+1}(\mu+\tfrac{1}{\tau_{k}}))^{\tfrac{1}{2}}\leq1,\ \  & \sum_{k=0}^{K-1}t_{k}(\nu_{k+1}(\mu+\tfrac{1}{\tau_{k}}))^{\tfrac{1}{2}}\leq1.\end{split}
\end{equation*}
\end{lem}
\begin{proof}
It follows from the updating rule in {\apd}
that 
\begin{equation*}
\tfrac{\tau_{k+1}}{\tau_{k}}=\sqrt{\tfrac{\gamma_{k}}{\gamma_{k+1}}}=\sqrt{\tfrac{\sigma_{k}}{\tau_{k}}\tfrac{\tau_{k+1}}{\sigma_{k+1}}}=\tfrac{1}{\sqrt{1+\mu\tau_{k}}}.
\end{equation*}
Then we have $\tfrac{\tau_{k+1}}{\tau_{k}}=\tfrac{\sigma_{k}}{\sigma_{k+1}}$,
which implies $t_{k}(\mu+\tfrac{1}{\tau_{k}})^{\tfrac{1}{2}}=\tfrac{t_{k+1}}{\sqrt{\tau_{k}}}$.
Then we have
\begin{equation}
\begin{split}\sum_{k=0}^{K-1}\tfrac{t_{k+1}}{\sqrt{\tau_{k}}}\sqrt{\delta_{k+1}} & \leq\sum_{k=0}^{K-1}\tfrac{1}{(k+2)^{2}}=\sum_{k=2}^{K+1}\tfrac{1}{k^{2}}\leq\int_{1}^{K+1}\tfrac{1}{s^{2}}ds\leq1.
\end{split}
\label{eq:stochastic_rate}
\end{equation}

If $\mu=0$, then $\tau_{k}=\tau_{0},t_{k}=1,\forall k\geq0$. Similarly,
we have the same conclusion like~\eqref{eq:stochastic_rate}. Moreover,
we have the same properties for sequence $\nu_{k}$. Then we complete
our proof.
\end{proof}

\section{Proof of Theorem~\ref{thm:inexact_thm_expected_appendix}}

\begin{proof}
    Take $\text{\ensuremath{\inprod{-\zbf_{k}}{\cdot}}}$ in Lemma~\ref{lem:three_point},
then we have
\begin{equation}
\begin{split} \inprod{\ybf-\ybf_{k+1}}{\zbf_{k}}
\leq  \tfrac{1}{2\sigma_{k}}\brbra{\norm{\ybf-\ybf_{k}}^{2}-\norm{\ybf-\ybf_{k+1}}^{2}-\norm{\ybf_{k+1}-\ybf_{k}}^{2}},\forall\ybf\in\mcal Y.
\end{split}
\label{eq:three_point_y-1}
\end{equation}
It follows from~\eqref{eq:Approximate_sol} that 
\begin{equation}
\begin{split} &\mcal L(\xbf_{k+1},\ybf)-\mcal L(\xbf,\ybf_{k+1})\\
= & f(\xbf_{k+1})+\inprod{\ybf-\ybf_{k+1}}{\gbf(\xbf_{k+1})}-f(\xbf)+\inprod{\ybf_{k+1}}{\gbf(\xbf_{k+1})-\gbf(\xbf)}\\
\leq & \tfrac{1}{\tau_{k}}\brbra{\mathbf{D}_{\Wbf}(\xbf,\xbf_{k})-\mathbf{D}_{\Wbf}(\xbf_{k+1},\xbf_{k})}\\
 & -(\tfrac{1}{\tau_{k}}+\mu)\mathbf{D}_{\Wbf}(\xbf,\xbf_{k+1})+2\sqrt{\tfrac{1}{\tau_{k}}+\mu}\sqrt{\mathbf{D}_{\Wbf}(\xbf,\xbf_{k+1})}\sqrt{\delta_{k+1}\mathbf{D}_{\Wbf}(\hat{\xbf}_{k+1},\xbf_{k})+\nu_{k+1}}\\
 & +\delta_{k+1}\mathbf{D}_{\Wbf}(\hat{\xbf}_{k+1},\xbf_{k})+\nu_{k+1}+\inprod{\ybf-\ybf_{k+1}}{\gbf(\xbf_{k+1})}.
\end{split}
\label{eq:conver-01-1}
\end{equation}
 By the definition $\zbf_{k}$, we have $
\zbf_{k}=\gbf(\xbf_{k+1})+\gbf(\xbf_{k})-\gbf(\xbf_{k+1})+\theta_{k}(\gbf(\xbf_{k})-\gbf(\xbf_{k-1}))$.
Let us denote $\qbf_{k}=\gbf(\xbf_{k})-\gbf(\xbf_{k-1})$ for brevity.
Then we have
\begin{equation}
\begin{split} & \inprod{\ybf-\ybf_{k+1}}{\gbf(\xbf_{k+1})}\\
= & \inprod{\ybf-\ybf_{k+1}}{\zbf_{k}}+\inprod{\ybf-\ybf_{k+1}}{\qbf_{k+1}}-\theta_{k}\inprod{\ybf-\ybf_{k}}{\qbf_{k}}-\theta_{k}\inprod{\ybf_{k}-\ybf_{k+1}}{\qbf_{k}}.
\end{split}
\label{eq:conver-02-1}
\end{equation}
Putting~\eqref{eq:three_point_y-1},~\eqref{eq:conver-01-1} and~\eqref{eq:conver-02-1}
together, we have
\begin{equation*}
\begin{split} & \mcal L(\xbf_{k+1},\ybf)-\mcal L(\xbf,\ybf_{k+1})\\
\leq & \tfrac{1}{\tau_{k}}\brbra{\mathbf{D}_{\Wbf}(\xbf,\xbf_{k})-\mathbf{D}_{\Wbf}(\xbf_{k+1},\xbf_{k})}\\
 & -(\tfrac{1}{\tau_{k}}+\mu)\mathbf{D}_{\Wbf}(\xbf,\xbf_{k+1})+2\sqrt{\tfrac{1}{\tau_{k}}+\mu}\sqrt{\mathbf{D}_{\Wbf}(\xbf,\xbf_{k+1})}\sqrt{\delta_{k+1}\mathbf{D}_{\Wbf}(\hat{\xbf}_{k+1},\xbf_{k})+\nu_{k+1}}\\
 & +\delta_{k+1}\mathbf{D}_{\Wbf}(\hat{\xbf}_{k+1},\xbf_{k})+\nu_{k+1}\\
 & +\inprod{\ybf-\ybf_{k+1}}{\qbf_{k+1}}-\theta_{k}\inprod{\ybf-\ybf_{k}}{\qbf_{k}}-\theta_{k}\inprod{\ybf_{k}-\ybf_{k+1}}{\qbf_{k}}\\
 & +\tfrac{1}{2\sigma_{k}}\brbra{\norm{\ybf-\ybf_{k}}^{2}-\norm{\ybf-\ybf_{k+1}}^{2}-\norm{\ybf_{k+1}-\ybf_{k}}^{2}}].
\end{split}
\end{equation*}

It follows from Young's inequality that 
\begin{equation}\label{eq:conver-10}
\begin{split}\inprod{\ybf_{k+1}-\ybf_{k}}{\qbf_{k}} & \leq\tfrac{\delta/\sigma_{k-1}}{2}\norm{\ybf_{k+1}-\ybf_{k}}^{2}+\tfrac{1}{2\delta/\sigma_{k-1}}\norm{\qbf_{k}}^{2}.\end{split}
\end{equation}
By $\tfrac{1}{2\delta/\sigma_{k}}\norm{\qbf_{k+1}}^{2}\leq\tfrac{L_{g}}{\delta/\sigma_{k}}\mathbf{D}_{\Wbf}(\xbf_{k+1},\xbf_{k})$ and~\eqref{eq:conver-10}, we have 
\begin{equation}
\begin{split} & \mcal L(\xbf_{k+1},\ybf)-\mcal L(\xbf,\ybf_{k+1})\\
\leq & \tfrac{1}{\tau_{k}}\mathbf{D}_{\Wbf}(\xbf,\xbf_{k})-\brbra{\tfrac{1}{\tau_{k}}+\mu}\mathbf{D}_{\Wbf}(\xbf,\xbf_{k+1})+(\tfrac{L_{g}}{\delta/\sigma_{k}}-\tfrac{1}{\tau_{k}})\mathbf{D}_{\Wbf}(\xbf_{k+1},\xbf_{k})\\
 & +\inprod{\ybf-\ybf_{k+1}}{\qbf_{k+1}}-\theta_{k}\inprod{\ybf-\ybf_{k}}{\qbf_{k}}+\brbra{\tfrac{\theta_{k}\delta/\sigma_{k-1}}{2}-\tfrac{1}{2\sigma_{k}}}\norm{\ybf_{k+1}-\ybf_{k}}^{2}\\
 & +\tfrac{1}{2\sigma_{k}}\brbra{\norm{\ybf-\ybf_{k}}^{2}-\norm{\ybf-\ybf_{k+1}}^{2}}+\tfrac{\theta_{k}}{2\delta/\sigma_{k-1}}\norm{\qbf_{k}}^{2}-\tfrac{1}{2\delta/\sigma_{k}}\norm{\qbf_{k+1}}^{2}+\eta_{k}.
\end{split}
\label{eq:conver-03-1}
\end{equation}
Multiply $t_{k}$ on both sides of~\eqref{eq:conver-03-1} and sum
up the result for $k=0,\ldots,K-1$. In view of the relation~\eqref{eq:parameter_setting},
we have
\begin{equation}
\begin{split} & \sum_{k=0}^{K-1}t_{k}[\mcal L(\xbf_{k+1},\ybf)-\mcal L(\xbf,\ybf_{k+1})]\\
\leq & t_{0}\tau_{0}^{-1}\mathbf{D}_{\Wbf}(\xbf,\xbf_{0})-t_{K-1}(\tau_{K-1}^{-1}+\mu)\mathbf{D}_{\Wbf}(\xbf,\xbf_{K})-\tfrac{t_{K-1}}{2\delta/\sigma_{K-1}}\norm{\qbf_{K}}^{2}\\
 & +\tfrac{t_{0}\sigma_{0}^{-1}}{2}\norm{\ybf-\ybf_{0}}^{2}-\tfrac{t_{K-1}\sigma_{K-1}^{-1}}{2}\norm{\ybf-\ybf_{K}}^{2}\\
 & +t_{K-1}\inprod{\ybf-\ybf_{K}}{\qbf_{K}}-t_{0}\theta_{0}\inprod{\ybf-\ybf_{0}}{\qbf_{0}}+\sum_{k=0}^{K-1}t_{k}\eta_{k}\\
\overset{(a)}{\leq} & t_{0}\tau_{0}^{-1}\mathbf{D}_{\Wbf}(\xbf,\xbf_{0})+\tfrac{t_{0}\sigma_{0}^{-1}}{2}\norm{\ybf-\ybf_{0}}^{2}-t_{K-1}(\tau_{K-1}^{-1}+\mu)\mathbf{D}_{\Wbf}(\xbf,\xbf_{K})+\sum_{k=0}^{K-1}t_{k}\eta_{k},
\end{split}
\label{eq:conver-04-1}
\end{equation}
where $(a)$ holds by $t_{K-1}\inprod{\ybf-\ybf_{K}}{\qbf_{K}}\leq t_{K-1}(\tfrac{1}{2\delta/\sigma_{K-1}}\norm{\qbf_{K}}^{2}+\tfrac{\delta/\sigma_{K-1}}{2}\norm{\ybf-\ybf_{K}}^{2})$
and $\qbf_{0}=\zerobf$. Since $\mcal L(\xbf,\ybf)$ is convex with
respect to $\xbf$ and linear in $\ybf$, then we have
\begin{equation}
T_{K}[\mcal L(\bar{\xbf}_{K,}\ybf)-\mcal L(\xbf,\bar{\ybf}_{K})]\leq\sum_{k=0}^{K-1}t_{k}[\mcal L(\xbf_{k+1},\ybf)-\mcal L(\xbf,\ybf_{k+1})].\label{eq:conver-05-1}
\end{equation}
Combining~\eqref{eq:conver-04-1} and~\eqref{eq:conver-05-1}, we
obtain~\eqref{eq:thm1_converge}.

\end{proof}

\section{Proof of Corollary~\ref{cor:inexact_cor_expected_appendix}}
\begin{proof}
First we need to verify all relations in~\eqref{eq:parameter_setting}
holds. It follows from the relation among $\tau_{k},\sigma_{k}$ that
\begin{equation*}
\tfrac{\tau_{k+1}}{\tau_{k}}=\tfrac{\sigma_{k}}{\sigma_{k+1}}\Rightarrow\tau_{k}\sigma_{k}=\tau_{0}\sigma_{0},\forall k\geq0.
\end{equation*}
Combining the above result and $\tau_{0}\sigma_{0}\leq\delta/L_{g}$,
we have the $L_{g}\sigma_{k}/\delta\leq(\tau_{k})^{-1}$. Furthermore,
the rest of~\eqref{eq:parameter_setting} can be easily verified
using the parameters updating rule in {\apd}.
Take $\xbf=\xbf^{*},\ybf=\ybf^{+}:=\frac{1}{\sqrt{m}}(\|\ybf^{*}\|_{1}+1)\frac{[\mathbf{g}(\bar{\xbf}_{K})]_{+}}{\|[\mathbf{g}(\bar{\xbf}_{K})]_{+}\|}$.
Note that $\mcal L(\bar{\xbf}_{K},\ybf^{*})-\mcal L(\xbf^{*},\ybf^{*})\geq0$,
which implies that $f(\bar{\xbf}_{K})+\inprod{\ybf^{*}}{\gbf(\bar{\xbf}_{K})}-f(\xbf^{*})\geq0$.
It follows from $\inprod{\ybf^{*}}{\gbf(\bar{\xbf}_{K})}\leq\norm{\ybf^{*}}\norm{[\gbf(\bar{\xbf}_{K})]_{+}}$
that 
\begin{equation}
f(\bar{\xbf}_{K})+\|\ybf^{*}\|\|[\mathbf{g}(\bar{\xbf}_{K})]_{+}\|-f(\xbf^{*})\geq0.\label{eq:cor_temp01}
\end{equation}

Moreover, we have 
\begin{equation*}
\begin{split} & \mathcal{L}(\bar{\xbf}_{K},\ybf^{+})-\mathcal{L}(\xbf^{*},\bar{\ybf}_{K})\\
\geq{} & \mathcal{L}(\bar{\xbf}_{K},\ybf^{+})-\mathcal{L}(\xbf^{*},\ybf^{*})\\
={} & f(\bar{\xbf}_{K})+\binner{\left(\|\ybf^{*}\|_{1}+1\right)\frac{[\mathbf{g}(\bar{\xbf}_{K})]_{+}}{\|[\mathbf{g}(\bar{\xbf}_{K})]_{+}\|}}{\mathbf{g}(\bar{\xbf}_{K})}-f(\xbf^{*})\\
\geq{} & f(\bar{\xbf}_{K})+\left(\|\ybf^{*}\|_{2}+1\right)\|[\mathbf{g}(\bar{\xbf}_{K})]_{+}\|-f(\xbf^{*})\\
\geq{} & \max\{f(\bar{\xbf}_{K})-f(\xbf^{*}),\|[\mathbf{g}(\bar{\xbf}_{K})]_{+}\|\},
\end{split}
\end{equation*}
where the second inequality is by $\|\ybf^{*}\|_{1}\geq\|\ybf^{*}\|_{2}$,
and $\left\langle [\mathbf{g}(\bar{\xbf}_{K})]_{+},\mathbf{g}(\bar{\xbf}_{K})\right\rangle =\|[\mathbf{g}(\bar{\xbf}_{K})]_{+}\|^{2}$,
and the last inequality is by~\eqref{eq:cor_temp01}. Hence, it follows
from~\eqref{eq:thm1_converge} that~\eqref{eq:cor_01_appendix} holds. 

Since $\sqrt{a+b}\leq\sqrt{a}+\sqrt{b}$ for $ab\geq0$ with $a=\delta_{k+1}\mathbf{D}_{\Wbf}(\hat{\xbf}_{k+1},\xbf_{k}),b=\nu_{k+1}$,
we have 
\begin{equation*}
\begin{split}\sum_{k=0}^{K-1}t_{k}\eta_{k} & \leq\sum_{k=0}^{K-1}\sqrt{2}D_{X}t_{k}\sqrt{\tfrac{1}{\tau_{k}}+\mu}(\tfrac{1}{\sqrt{2}}D_{X}\sqrt{\delta_{k+1}}+\sqrt{\nu_{k+1}})\end{split}
,
\end{equation*}
where $D_{X}:=\max_{\xbf_{1},\xbf_{2}}\sqrt{2\mathbf{D}_{\Wbf}(\xbf_{1},\xbf_{2})}$.
It follows from Lemma~\ref{lem:stochastic_error} that $\sum_{k=0}^{K-1}t_{k}\eta_{k}\leq D_{X}^{2}+\sqrt{2}D_{X}$.
Hence, we can get $\tfrac{\Delta_{*+}+D_{X}^{2}+\sqrt{2}D_{X}}{T_{k}}\leq\vep$
when $K\geq\sqrt{\tfrac{6(\mu\tau_{0})^{-1}(\Delta_{*+}+D_{X}^{2}+\sqrt{2}D_{X})}{\vep}-6(\mu\tau_{0})^{-1}}+1$.
If $\mu=0$, then we have $T_{k}=K$, then we can obtain $\vep$-optimal
solution when $K\geq\tfrac{\Delta_{*+}+D_{X}^{2}+\sqrt{2}D_{X}}{\vep}$.
\end{proof}

\section{Proof of Theorem~\ref{thm:3.2_appendix}}

\begin{lem}
\label{lem:2.5} Let $\hat{\xbf}^{t+1}$ be the optimal solution for
the sub-problem~\eqref{eq:lcp_subprob_appendix} with $L>\max\bcbra{\rho,{\sigma_{\max}(\mathbf{W}^{\top} \mathbf{W})}^{-1}}$,
where $\rho=\max_{i}\bcbra{L_{h_{i}}}$. Then we have $\xbf^{t}$
is a $(\vep,\vep)$-FJ point if $\norm{\hat{\xbf}^{t+1}-\xbf^{t}}\leq\tfrac{\vep}{L\sigma_{\max}(\mathbf{W}^{\top} \mathbf{W})}$.
\end{lem}

\begin{proof}
It follows from the definition of $\hat{\xbf}^{t+1}$ that $\hat{\xbf}^{t+1}$
satisfying FJ condition~\eqref{eq:fj_point_1_appendix} and~\eqref{eq:fj_point_2_appendix}.
Then there exists $y_{t0}\geq0,y_{t}\geq0,y_{t0}+y_{t}=1,\nu\in\mcal N_{\mcal X}(\hat{\xbf}^{t+1}), d_f(\hat{\xbf}^{t+1})\in \partial f(\hat{\xbf}^{t+1}), d_{\bar{\phi}}(\hat{\xbf}^{t+1})\in \partial \bar{\phi}(\hat{\xbf}^{t+1})$
such that 
\begin{equation*}
y_{t0}(d_f(\hat{\xbf}^{t+1})+L\mathbf{W}^{\top} \mathbf{W}(\hat{\xbf}^{t+1}-\xbf^{t}))+y_{t}(d_{\bar{\phi}}(\hat{\xbf}^{t+1})+L\mathbf{W}^{\top} \mathbf{W}(\hat{\xbf}^{t+1}-\xbf^{t}))=-\nu.
\end{equation*}
When $\norm{\hat{\xbf}^{t+1}-\xbf^{t}}\leq\tfrac{\vep}{L\sigma_{\max}(\mathbf{W}^{\top} \mathbf{W})}$,
we have
\begin{equation}
\norm{y_{t0}(d_f(\hat{\xbf}^{t+1})+y_{t}(d_{\bar{\phi}}(\hat{\xbf}^{t+1})+\nu}=\norm{L\mathbf{W}^{\top} \mathbf{W}(\hat{\xbf}^{t+1}-\xbf^{t})}\leq\vep.\label{eq:fj2}
\end{equation}
When $y_{t}=0,$ it is easy to obtain $\abs{y_{k}\bar{\phi}(\hat{\xbf}^{t+1})}=0$,
thus we only consider the case that $y_{k}$ is positive. If $y_{k}\geq0$,
then we have $\phi_{t}(\hat{\xbf}^{t+1})=0$, then we have
\begin{equation*}
0\geq\bar{\phi}(\hat{\xbf}^{t+1})=-L\mathbf{D}_{\Wbf}(\hat{\xbf}^{t+1},\xbf^{t})\geq-\tfrac{1}{2}L\sigma_{\max}(\mathbf{W}^{\top} \mathbf{W})\norm{\hat{\xbf}^{t+1}-\xbf^{t}}^{2}\geq-\tfrac{\vep^{2}}{2L\sigma_{\max}(\mathbf{W}^{\top} \mathbf{W})}.
\end{equation*}
 Therefore, we can obtain
\begin{equation}
\abs{y_{k}\bar{\phi}(\hat{\xbf}^{t+1})}\leq\abs{\bar{\phi}(\hat{\xbf}^{t+1})}\leq\tfrac{\vep^{2}}{2L\sigma_{\max}(\mathbf{W}^{\top} \mathbf{W})}\leq\vep^{2}.\label{eq:fj1}
\end{equation}
Combining~\eqref{eq:fj1} and~\eqref{eq:fj2}, we have $\hat{\xbf}^{t+1}$
is an $\vep$-FJ point of problem~\eqref{eq:fair_prob_appendix}. Due to $\norm{\hat{\xbf}^{t+1}-\xbf^{t}}\leq\tfrac{\vep}{L\sigma_{\max}(\mathbf{W}^{\top} \mathbf{W})}\leq\vep$,
$\xbf^{t}$ is an $(\vep,\vep)$-FJ point for problem~\eqref{eq:fair_prob_appendix}.
\end{proof}

\begin{lem}
\label{lem:3.4} Suppose $\vep_{3}=\vep_{4}=\tfrac{(L-\rho)\vep^{2}\sigma_{\min}(\mathbf{W}^{\top} \mathbf{W})}{4L^{2}\sigma_{\max}(\mathbf{W}^{\top} \mathbf{W})^{2}}$
in {\ilcp}, then we have {\ilcp}
has $\bar{\phi}(\xbf^{t})\leq0$ before $\xbf^{t}$ becomes an $(\vep,\vep)$-FJ
point.
\end{lem}
\begin{proof}
We show the iterative sequence $\{\xbf^{t}\}$ is feasible, i.e.,
$\bar{\phi}(\xbf^{t})\leq0$ before algorithm reaches $(\vep,\vep)$-FJ
point by induction. Assume $\phi_{t}(\xbf^{t})=\bar{\phi}(\xbf^{t})\leq0$.
Then there exists $d_{\phi_t^0}(\hat{\xbf}^{t+1})\in \partial \phi_t^{0}(\hat{\xbf}^{t+1})$ and $d_{\bar{\phi}_t}(\hat{\xbf}^{t+1})\in \partial \bar{\phi}_t(\hat{\xbf}^{t+1})$ such that
\begin{equation}
\begin{split}y_{t0}\phi_{t}^{0}(\xbf^{t+1})+y_{t}\bar{\phi}_{t}(\xbf^{t+1}) & \geq\inprod{y_{t0}d_{\phi_t^0}(\hat{\xbf}^{t+1})+y_{k}d_{\bar{\phi}_t}(\hat{\xbf}^{t+1})}{\xbf^{t+1}-\hat{\xbf}^{t+1}}\\
 & \ \ +y_{t0}\phi_{t}^{0}(\hat{\xbf}^{t+1})+y_{t}\bar{\phi}_{t}(\hat{\xbf}^{t+1})+(L-\rho)\mathbf{D}_{\Wbf}(\xbf^{t+1},\hat{\xbf}^{t+1})
\end{split}
\label{eq:temp01}
\end{equation}
Since $y_{t0}d_{\phi_t^0}(\hat{\xbf}^{t+1})+y_{t}d_{\bar{\phi}_t}(\hat{\xbf}^{t+1})\in-\mcal N_{\mcal X}(\hat{\xbf}^{t+1})$
by FJ condition, and due to $\xbf^{t+1}\in\mcal X$, we have
\begin{equation}
(y_{t0}d_{\phi_t^0}(\hat{\xbf}^{t+1})+y_{t}d_{\bar{\phi}_t}(\hat{\xbf}^{t+1}))^{T}(\xbf^{t+1}-\hat{\xbf}^{t+1})\geq0.\label{eq:temp02}
\end{equation}
Furthermore, by $y_{t}\bar{\phi}_{t}(\hat{\xbf}^{t+1})=0$ by FJ conditions,
and $\xbf^{t+1}$ being an $(\vep_{3},\vep_{4})$-optimal solution
for the subproblem yields $\phi_{t}^{0}(\xbf^{t+1})-\phi_{t}^{0}(\hat{\xbf}^{t+1})\leq\vep_{3}$
and $\bar{\phi}_{t}(\xbf^{t+1})\leq\vep_{4}$, then combining~\eqref{eq:temp01}
and~\eqref{eq:temp02}, we have
\begin{equation*}
(L-\rho)\mathbf{D}_{\Wbf}(\xbf^{t+1},\hat{\xbf}^{t+1})\leq y_{t0}\vep_{3}+y_{t}\vep_{4}.
\end{equation*}
It follows from Lemma~\ref{lem:2.5} that before reaching the $(\vep,\vep)$-
FJ point, we have $\norm{\hat{\xbf}^{t+1}-\xbf^{t}}\geq\tfrac{\vep}{L\sigma_{\max}(\mathbf{W}^{\top} \mathbf{W})},$
which implies ${\mathbf{D}_{\Wbf}(\hat{\xbf}^{t+1},\xbf^{t})\geq\tfrac{\sigma_{\min}(\mathbf{W}^{\top} \mathbf{W})\vep^{2}}{L^{2}\sigma_{\max}(\mathbf{W}^{\top} \mathbf{W})^{2}}}$.
Therefore, we can obtain
\begin{equation*}
\mathbf{D}_{\Wbf}(\xbf^{t+1},\xbf^{t})\geq\tfrac{1}{2}\mathbf{D}_{\Wbf}(\hat{\xbf}^{t+1},\xbf^{t})-\mathbf{D}_{\Wbf}(\hat{\xbf}^{t+1},\xbf^{t+1})\geq\tfrac{1}{2}\tfrac{\sigma_{\min}(\mathbf{W}^{\top} \mathbf{W})\vep^{2}}{L^{2}\sigma_{\max}(\mathbf{W}^{\top} \mathbf{W})^{2}}-\tfrac{y_{t0}\vep_{3}+y_{t}\vep_{4}}{(L-\rho)}
\end{equation*}

By the definition of $\vep_3$ and $\vep_4$, we have
\begin{equation*}
\mathbf{D}_{\Wbf}(\xbf^{t+1},\xbf^{t})>\tfrac{\sigma_{\min}(\mathbf{W}^{\top} \mathbf{W})\vep^{2}}{2L^{2}\sigma_{\max}(\mathbf{W}^{\top} \mathbf{W})^{2}}-\tfrac{(L-\rho)\sigma_{\min}(\mathbf{W}^{\top} \mathbf{W})\vep^{2}}{4L^{2}\sigma_{\max}(\mathbf{W}^{\top} \mathbf{W})^{2}}\cdot\tfrac{1}{(L-\rho)}=\tfrac{\sigma_{\min}(\mathbf{W}^{\top} \mathbf{W})\vep^{2}}{4L^{2}\sigma_{\max}(\mathbf{W}^{\top} \mathbf{W})^{2}},
\end{equation*}
which implies that 
\begin{align*}
\bar{\phi}(\xbf^{t+1}) & =\bar{\phi}_{t}(\xbf^{t+1})-L\mathbf{D}_{\Wbf}(\xbf^{t+1},\xbf^{t})\\
 & \leq\tfrac{(L-\rho)\sigma_{\min}(\mathbf{W}^{\top} \mathbf{W})\vep^{2}}{4L^{2}\sigma_{\max}(\mathbf{W}^{\top} \mathbf{W})}-\tfrac{\sigma_{\min}(\mathbf{W}^{\top} \mathbf{W})\vep^{2}}{4L\sigma_{\max}(\mathbf{W}^{\top} \mathbf{W})^{2}}=-\tfrac{\sigma_{\min}(\mathbf{W}^{\top} \mathbf{W})\rho\vep^{2}}{4L^{2}\sigma_{\max}(\mathbf{W}^{\top} \mathbf{W})}<0.\nonumber 
\end{align*}
Now, we complete our induction.
\end{proof}

\begin{proof}[Main Proof of Theorem~\ref{thm:3.2_appendix}]
    It follows from Lemma~\ref{lem:3.4} that the iterative sequence
$\{\xbf^{t}\}$ are always feasible, i.e., $\bar{\phi}(\xbf^{t})\leq0$
for the problem~\eqref{eq:fair_prob_appendix} before we reach an $(\vep,\vep)$-FJ
point. For any $\hat{\xbf}^{t+1},y_{t0},y_{t}$ satisfy FJ condition
of sub-problem~\eqref{eq:lcp_subprob_appendix}, we have there exist $d_{\phi_t^0}(\hat{\xbf}^{t+1})\in \partial \phi_t^{0}(\hat{\xbf}^{t+1})$ and $d_{\bar{\phi}_t}(\hat{\xbf}^{t+1})\in \partial \bar{\phi}_t(\hat{\xbf}^{t+1})$ such that 
\begin{align}
y_{t0}d_{\phi_t^0}(\hat{\xbf}^{t+1})+y_{t}d_{\bar{\phi}_t}(\hat{\xbf}^{t+1}) & \in-\mcal N_{\mcal X}(\hat{\xbf}^{t+1}),\label{eq:optimal_sol}\\
y_{t}\bar{\phi}_{t}(\hat{\xbf}^{t+1}) & =0.\label{eq:cross0}
\end{align}
It follows from~\eqref{eq:optimal_sol} and $\xbf^{t}\in\mcal X$
that 
\begin{equation}
\brbra{y_{t0}d_{\phi_t^0}(\hat{\xbf}^{t+1})+y_{t}d_{\bar{\phi}_t}(\hat{\xbf}^{t+1})}^{T}(\xbf^{t}-\hat{\xbf}^{t+1})\geq0.\label{eq:opt_sol}
\end{equation}
 Note the following inequality
\begin{equation}
\begin{split}y_{t0}\phi_{t}^{0}(\xbf^{t})+y_{t}\bar{\phi}_{t}(\xbf^{t}) & \geq\inprod{y_{t0}d_{\phi_t^0}(\hat{\xbf}^{t+1})+y_{t}d_{\bar{\phi}_t}(\hat{\xbf}^{t+1})}{\xbf^{t}-\hat{\xbf}^{t+1}}\\
 & \ \ +y_{t0}\phi_{t}^{0}(\hat{\xbf}^{t+1})+y_{t}\bar{\phi}_{t}(\hat{\xbf}^{t+1})+(L-\rho)\mathbf{D}_{\Wbf}(\xbf^{t},\hat{\xbf}^{t+1}),
\end{split}
\label{eq:temp01-1}
\end{equation}
holds by $y_{t0}\phi_{t}^{0}(\xbf)+y_{t}\bar{\phi}_{t}(\xbf)$ is
$L-\rho$ relative strongly convex, where $\hat{\xbf}^{t+1}$ be the
exact solution for sub-problem~\eqref{eq:lcp_subprob_appendix}. Since $\bar{\phi}_{t}(\xbf^{t})=\bar{\phi}(\xbf^{t})\leq0$
from Lemma~\ref{lem:3.4},~\eqref{eq:opt_sol} and~\eqref{eq:cross0},
we have
\begin{equation}
y_{t0}\phi_{t}^{0}(\xbf^{t})\geq y_{t0}\phi_{t}^{0}(\hat{\xbf}^{t+1})+(L-\rho)\mathbf{D}_{\Wbf}(\xbf^{t},\hat{\xbf}^{t+1}).\label{eq:temp03}
\end{equation}
Since $\xbf^{t+1}$ is a $(\vep_{3},\vep_{4})$- optimal solution
for sub-problem~\eqref{eq:lcp_subprob_appendix}, we have $\phi_{t}^{0}(\xbf^{t+1})-\phi_{t}^{0}(\hat{\xbf}^{t+1})\leq\vep_{3}$.
Then~\eqref{eq:temp03} can be further written as follows:
\begin{equation*}
y_{t0}f(\xbf^{t})\geq y_{t0}(f(\xbf^{t+1})-\vep_{3})+(L-\rho)\mathbf{D}_{\Wbf}(\xbf^{t},\hat{\xbf}^{t+1}).
\end{equation*}
Rearrange the above inequality, we have 
\begin{equation}
y_{t0}(f(\xbf^{t})-f(\xbf^{t+1}))\geq(L-\rho)\mathbf{D}_{\Wbf}(\xbf^{t},\hat{\xbf}^{t+1})-y_{t0}\vep_{3}.\label{eq:temp04}
\end{equation}

When $y_{t0}=0$, then $\mathbf{D}_{\Wbf}(\xbf^{t},\hat{\xbf}^{t+1})=0$ and
we attain an exact stationary point $\xbf^{t}$ for problem~\eqref{eq:fair_prob_appendix}.
Now, for the case $y_{t0}>0$, it follows from Lemma~\ref{lem:2.5} that before reaching the $(\vep,\vep)$-FJ
point $\hat{\xbf}^{t+1}$, we have
\begin{equation*}
{\mathbf{D}_{\Wbf}(\hat{\xbf}^{t+1},\xbf^{t})\geq\sigma_{\min}(\mathbf{W}^{\top} \mathbf{W})}\brbra{\tfrac{\vep}{L\sigma_{\max}(\mathbf{W}^{\top} \mathbf{W})}}^{2}.
\end{equation*}
Thus, before reaching the $(\vep,\vep)$-FJ point $\xbf^{t}$,
it follows from~\eqref{eq:temp04} that 
\begin{equation}\label{eq:key_ineq}
\begin{split}f(\xbf^{t})-f(\xbf^{t+1}) & \geq\tfrac{(L-\rho)\mathbf{D}_{\Wbf}(\xbf^{t},\hat{\xbf}^{t+1})-y_{t0}\vep_{3}}{y_{t0}}\\
 & \geq\tfrac{(L-\rho)\mathbf{D}_{\Wbf}(\xbf^{t},\hat{\xbf}^{t+1})}{y_{t0}}-\vep_{3}\\
 & \geq(L-\rho)\mathbf{D}_{\Wbf}(\xbf^{t},\hat{\xbf}^{t+1})-\vep_{3}\\
 & \geq\tfrac{(L-\rho)\sigma_{\min}(\mathbf{W}^{\top} \mathbf{W})\vep^{2}}{L^{2}\sigma_{\max}(\mathbf{W}^{\top} \mathbf{W})^{2}}-\tfrac{(L-\rho)\sigma_{\min}(\mathbf{W}^{\top} \mathbf{W})\vep^{2}}{4L^{2}\sigma_{\max}(\mathbf{W}^{\top} \mathbf{W})^{2}}\\
 & =\tfrac{3(L-\rho)\sigma_{\min}(\mathbf{W}^{\top} \mathbf{W})\vep^{2}}{4L^{2}\sigma_{\max}(\mathbf{W}^{\top} \mathbf{W})^{2}},
\end{split}
\end{equation}
Suppose $\xbf^T$ is the $(\vep,\vep)$-FJ point, then we have~\eqref{eq:key_ineq} holds for any $t=0,\ldots,T$. Summing~\eqref{eq:key_ineq} over  $0,\ldots,T$, we have
\begin{equation*}
T\leq\tfrac{4L^{2}\sigma_{\max}(\mathbf{W}^{\top} \mathbf{W})^{2}(f(\xbf^{0})-\inf_{\xbf\in\mcal X}f(\xbf))}{3(L-\rho)\sigma_{\min}(\mathbf{W}^{\top} \mathbf{W})\vep^{2}}-1.
\end{equation*}
Then, combining $\mu = L>0$ and the complexity result in Corollary~\ref{cor:inexact_cor_expected_appendix}, we complete our proof.
\end{proof}

\section{\label{sec:Heuristic-Mechansim}Heuristic Mechanism}

In order to bridge the gap between the original NP problem and its 
convex relaxation, we add a hyper-parameter to control the influence
of outliers. Specifically, consider that there exists a constraint
$\text{Pr}_{\abf\mid b=1}(\phi(\abf)\neq1)\leq0.01$ for the NPC problem.
Then the corresponding constraint of cross-entropy loss optimization
in our setting is $\sum_{i}^{n}\mathbb{I}(b{}_{i}=1)\log(F(\abf_{i}))\leq-n_{1}\log(1-0.01)\approx0.01n_{1}$.
And we assume $n_{1}=300$ for illustration. In most cases, this expectations
on the training set can help us get a good training error. However,
if there exists an outlier $\{\abf_{j},b_{j}=1\}$ such that $F(\abf_{j})=0.01$.
Hence, an outlier causes a loss term $\log(F(\abf_{j}))\approx4.6\geq0.01\cdot300$.
This makes the convex relaxation problem infeasible while the original
NP problem is feasible. So we artificially set a threshold to shrink
the predictions to a value that deviates too much. Experimental results
show that this heuristic can significantly improve the performance
of our algorithm.